\documentclass[12pt]{article}
\usepackage{amsmath, bm}
\usepackage{amssymb, amsthm}
\usepackage{algorithm2e}
\usepackage{graphicx}
\usepackage{hyperref}
\usepackage{setspace}
\usepackage{subcaption}
\usepackage{natbib}
\usepackage{url} 
\usepackage{enumitem}
\usepackage{xcolor}
\usepackage{authblk}
\usepackage{comment}
\usepackage{bbm}
\usepackage{float}
\usepackage{rotating}
\usepackage{geometry}
\usepackage{pdflscape}
\newcommand{\blind}{0}
\newtheorem{theorem}{Theorem}[section]

\newcommand{\numbereqn}{\addtocounter{equation}{1}\tag{\theequation}} 

\newcommand{\argmax}{\operatornamewithlimits{argmax}}
\newcommand{\argmin}{\operatornamewithlimits{argmin}}

\newcommand{\Tm}{{\Theta_m}}
\newcommand{\softmax}{{\mathrm{softmax}}}
\newcommand{\cM}{\mathcal{M}}

\begin{document}

\def\spacingset#1{\renewcommand{\baselinestretch}%
{#1}\small\normalsize} \spacingset{1}

\date{}
\if0\blind
{ 
  \title{\bf Neural-$g$: A Deep Learning Framework for Mixing Density Estimation \thanks{The authors gratefully acknowledge support from the National Science Foundation grants DMS-2015528 (RB, SW) and DMS-2112887 (QQ).}}
  
\author{Shijie Wang$^a$, Saptarshi Chakraborty$^b$, Qian Qin$^c$, Ray Bai$^a$  \vspace{0.2cm}\\ 
    $^a$Department of Statistics, University of South Carolina, Columbia, SC 29208 \\ 
    $^b$Department of Biostatistics, University at Buffalo, Buffalo, NY 14214 \\ 
    $^c$School of Statistics, University of Minnesota, Minneapolis, MN 55455\\}

  \maketitle
} \fi

\if1\blind
{
  \begin{center}
    {\LARGE\bf Title}
\end{center}
} \fi

\begin{abstract}
Mixing (or prior) density estimation is an important problem in statistics, especially in empirical Bayes $g$-modeling where accurately estimating the prior is necessary for making good posterior inferences. In this paper, we propose \emph{neural-$g$}, a new neural network-based estimator for $g$-modeling. Neural-$g$ uses a softmax output layer to ensure that the estimated prior is a valid probability density. Under default hyperparameters, we show that neural-$g$ is very flexible and capable of capturing many unknown densities, including those with flat regions, heavy tails, and/or discontinuities. In contrast, existing methods struggle to capture all of these prior shapes. We provide justification for neural-$g$ by establishing a new universal approximation theorem regarding the capability of neural networks to learn arbitrary probability mass functions. To accelerate convergence of our numerical implementation, we utilize a weighted average gradient descent approach to update the network parameters. Finally, we extend neural-$g$ to multivariate prior density estimation. We illustrate the efficacy of our approach through simulations and analyses of real datasets. A software package to implement neural-$g$ is publicly available at \url{https://github.com/shijiew97/neuralG}.


\end{abstract}

\noindent%

\spacingset{1.5} 

\newpage 

\section{Introduction}\label{sec:intro}

Mixture models have been applied extensively to many statistical procedures, including random effects and mixed effects models \citep{MagderZeger1996}, latent class models \citep{Oberski2016}, measurement error models \citep{RoederCarrollLindsay1996}, and empirical Bayes \citep{EfronBradley2010LIEB, KoenkerGu2017}. See \cite{lindsay1995mixture} for an overview of applications of mixture models. Suppose that we observe $n$ independent observations $\bm{y} = (y_1, \ldots, y_n)$, where the $y_i$'s are drawn from a known conditional density $f(y_i \mid \theta_i)$ but the $\theta_i$'s arise from an unknown latent density $\pi$. That is,
\begin{equation}\label{eq:mixturemodel}
    y_i\mid\theta_i\sim{f}(y_i\mid\theta_i),~~\theta_i \sim \mathrm{\pi}(\theta),~~~~ i = 1, \ldots, n, 
\end{equation} 
where the mixing density (or prior) $\pi$ is unknown. The problem of estimating $\pi$ has been referred to as $g$-modeling in the literature \citep{Efron2014modeling, efron2016empirical}, and there is an extensive literature on ways to estimate the prior \citep{kiefer1956consistency, laird1978nonparametric, Morris1983, Efron2014modeling, efron2016empirical}. Accurate estimation of the prior is important for making good infererences about the Bayes estimator $\mathbb{E}[\theta \mid \bm{y}]$ and other functionals of the posterior $\pi( \theta \mid \bm{y})$  \citep{Efron2014modeling, efron2016empirical}.

A classic approach to $g$-modeling is the nonparametric maximum likelihood estimator (NPMLE), originally introduced by \cite{kiefer1956consistency}. In NPMLE, inference for $\mathrm{\pi}$ proceeds by maximizing the complete likelihood $ \prod_{i=1}^n \int_{\Theta} f(y_i\mid \theta_i) \pi(\theta_i) d \theta_i$, where $\Theta$ denotes the parameter space for the $\theta_i$'s. Let $\Pi_{\Theta}$ denote the class of all probability densities defined on $\Theta$. The NPMLE is defined as
\begin{equation}\label{eq:kwnpmle}
\widehat{\pi} = \underset{\pi \in \Pi_{\Theta}}{\mathrm{argmax}}~\frac{1}{n} \sum_{i=1}^n \log \bigg\{ \int_{\Theta} f(y_i\mid \theta_i ) \pi(\theta_i) d \theta_i \bigg\},
\end{equation}
\cite{lindsay1995mixture} proved that under some mild conditions, the solution $\widehat{\pi}$ to \eqref{eq:kwnpmle} exists, is unique, and is almost surely discrete with support of at most $n$ atoms. However, since the atoms for $\widehat{\pi}$ and the true $\Theta$ are typically unknown, it is necessary to discretize the problem \eqref{eq:kwnpmle} over a grid of $m$ known points for most practical implementations of NPMLE. In practice, we often assume that $\pi$ is defined on a finite set $\Theta_m = \{ \theta_1, \ldots, \theta_m \}$, where $m \leq n$ and $\theta_1 <  \cdots < \theta_m$ are specified in advance, usually in the range from $y_{\min}$ to $y_{\max}$.


With the specification of the set $\Theta_m$, the NPMLE solution can be determined by solving
\begin{equation} \label{eq:kwnpmle_sol}
    \widehat{\bm{w}} = \argmax_{\bm{w} \in S_{m-1}}~\frac{1}{n} \sum_{i=1}^{n} \log \left\{ \sum_{j=1}^{m} f(y_i \mid \theta_j) w_j  \right\},
\end{equation}
where $\bm{w} = (w_1, \ldots, w_m)^\top$ is a vector of mixture probabilities and $\mathcal{S}_{m-1} = \{ \bm{w} : w_j \geq 0, j = 1, \ldots, m,~ \sum_{j=1}^{m} w_j = 1 \}$ denotes the $(m-1)$-simplex. The NPMLE estimator \eqref{eq:kwnpmle} is then taken to be $\widehat{\pi}_{\text{NPMLE}} = \sum_{j=1}^{m} \widehat{w}_j \delta_{\theta_j}$, where $\delta_{\theta_j}(\theta) = 1$ if $\theta = \theta_j$ and $\delta_{\theta_j}(\theta) = 0$ otherwise. \citet{koenker2014convex} observed that \eqref{eq:kwnpmle_sol} has an equivalent dual convex formulation which can be solved with the interior point method. 

Despite efficient methods for solving \eqref{eq:kwnpmle_sol}, it has been observed that in practice, the NPMLE tends to be quite coarse. In other words, the NPMLE \eqref{eq:kwnpmle} tends to fit mixture models where the number of atoms is \emph{much} smaller than $n$. \cite{koenker2014convex} suggested that the number of nonzero mixture components is typically $O(\sqrt{n})$. Later, \cite{polyanskiy2020} provided theoretical justification for this empirical observation by establishing the self-regularizing property of the NPMLE. They proved that if the density $f( \cdot \mid \theta )$ in \eqref{eq:mixturemodel} belongs to the exponential family, then the number of atoms in the NPMLE \eqref{eq:kwnpmle} can actually be upper bounded by $O(\log n)$. 

The coarseness of the NPMLE can be a practical hindrance, especially in empirical Bayes applications where a finer representation of the prior $\pi(\theta)$ is often desirable for posterior inference. For example, under a discrete prior with only a few atoms, posterior credible intervals for $\pi(\theta \mid \bm{y})$ are often too narrow and overly confident \citep{koenker2020empirical}. As a result, several authors have introduced smoothed variants of the NPMLE based on kernel smoothing, penalized splines, boosting, or bootstrapping \citep{Silverman1990JRSSB, liu2009functional, li2021boosting, Wang2024GBNPMLE}. However, these smoothed NPMLE methods remain in the general nonparametric framework and assume \textit{a priori} that the prior $\pi$ in \eqref{eq:mixturemodel} is smooth. In the absence of prior knowledge about the characteristics of the true mixing density, it may be desirable to have a flexible method that can estimate either a non-smooth \emph{or} a smooth prior. \cite{efron2016empirical} also pointed out that nonparametric $g$-modeling approaches often suffer from a lack of efficiency, with particularly slow rates of convergence.

As an alternative to nonparametric estimation, one can instead restrict attention to a parametric family \citep{Morris1983, efron2016empirical}. That is, instead of maximizing the total likelihood over the space of \emph{all} probability densities on $\Theta$ as in \eqref{eq:kwnpmle}, one can instead maximize $\pi$ with respect to a parametric family of densities. This is the approach taken by \cite{efron2016empirical}, henceforce referred to as \textit{Efron's $g$}. Similar to \eqref{eq:kwnpmle_sol}, \cite{efron2016empirical} assumes that the parameter space $\Theta_m = \{ \theta_1, \ldots, \theta_m \}$ is known, which is necessary for practical implementation. Despite being a discrete estimator, Efron's $g$ can still be quite smooth in practice. Setting the grid size $m$ to be arbitrarily large can produce extremely smooth estimates for Efron's $g$.
More specifically, \cite{efron2016empirical} proposed modeling the unknown mixing density $\pi$ in \eqref{eq:mixturemodel} using a curved exponential family \citep{EfronCurved1975} with support $\Theta_{m} =\{\theta_1,\ldots,\theta_m \}$ as
\begin{equation}\label{eq:efron-npmle}
 \pi(\bm{\alpha}) = \text{exp}\{ \bm{Q} \bm{\alpha} -\phi(\bm{\alpha})\} \text{ and } \pi_j(\bm{\alpha}) = \text{exp}\{ \bm{Q}_j^{\top} \bm{\alpha}-\phi(\bm{\alpha})\},~~ j = 1, \ldots, m,
\end{equation}
where $\bm{\alpha} \in \mathbb{R}^p$ is the unknown parameter of interest, $\bm{Q}$ is a known $m\times p$ matrix of basis functions, and $\phi(\bm{\alpha}) =\log\sum_{j=1}^m\exp(\bm{Q}_j^{\top} \bm{\alpha})$ is the normalizing factor. 
Under \eqref{eq:efron-npmle}, \cite{efron2016empirical} first estimates
\begin{equation}\label{eq:efron-npmle2}
\widehat{\bm{\alpha}} = \underset{\bm{\alpha}}{\mathrm{argmax}} \sum_{i=1}^n \log \bigg\{ { \sum_{j=1}^m} \  f(y_i \mid \theta_j) \pi_j(\bm{\alpha})\bigg\} + \lambda \left( {\textstyle \sum_{h=1}^p}\alpha_h^2 \right)^{1/2}
\end{equation}
where the tuning parameter $\lambda$ controls the degree of regularization. As recommended in \cite{efron2016empirical}, the additional regularization term $\sum_{h=1}^p \alpha_h^2$ in \eqref{eq:efron-npmle2} helps to prevent overfitting. Once the solution $\widehat{\bm{\alpha}}$ to \eqref{eq:efron-npmle2} has been obtained, Efron's $g$ is straightforwardly estimated as $\widehat{\pi}_{\text{Efron}} = \sum_{j=1}^{m} \pi_j(\widehat{\bm{\alpha}}) \delta_{\theta_j}$.

\cite{efron2016empirical} suggests a default choice of a natural cubic spline matrix with $p=5$ degrees of freedom for $\bm{Q}$ and $\lambda=1$ or $\lambda=2$ in \eqref{eq:efron-npmle2}. Despite the appeals and the flexibility of Efron's $g$, there are several limitations to this approach. First, the results for Efron's $g$ are highly sensitive to the choices of the penalty parameter $\lambda$ and number of basis functions $p$ \citep{koenker2019comment}. There is also currently a lack of guidance on how to properly tune these parameters. A grid search over $(\lambda, p)$ seems plausible but this requires refitting Efron's $g$ for many different pairs of $(\lambda, p)$, and it remains unclear what criterion should be used to select an ``optimal'' choice of $(\lambda, p)$. On the other hand, the original NPMLE optimization \eqref{eq:kwnpmle_sol} is completely free of such tuning parameters. 

A second issue is that Efron's $g$ may not be well-suited for estimating priors that are not in the curved exponential family \citep{EfronCurved1975}. For example, uniform densities, piecewise constant densities, and a number of heavy-tailed densities all fall outside of the parametric family used by Efron's $g$. However, these types of priors can be very useful in applications. For instance, in survival analysis, a piecewise constant prior is often used to model the baseline hazard function \citep{IbrahimSpringer2001}. This is because in practice, the follow-up time for the study is often split into several intervals \citep{NingCastillo2024}. By allowing the baseline hazard to take a constant value on each sub-interval, clinicians can easily interpret the model output as the underlying risk for each study period in the absence of covariates \citep{NingCastillo2024}. Heavy-tailed priors also play an important role in small area estimation and clinical trials, helping to alleviate the influence of outliers and conflicts between prior information and the data \citep{CookFuquenePericchi2009, BellHuang2006, DattaLahiri1995}. However, as noted, it is difficult for both Efron's $g$ and NPMLE to estimate these types of practically used priors.

The main thrust of this paper is to introduce a flexible procedure for $g$-modeling that strikes a balance between fully nonparametric $g$-modeling (e.g. NPMLE) and fully parametric $g$-modeling (e.g. Efron's $g$). Specifically, we introduce a new method for $g$-modeling based on deep neural networks. Our method, which we call \textit{neural-$g$}, uses a softmax output layer to estimate a valid probability mass function (PMF) over $\Theta_m$. Under a set of default hyperparameters (i.e. the number of hidden layers and number of nodes per hidden layer), neural-$g$ is able to capture a wide range of priors, including both smooth \emph{and} non-smooth densities. Neural-$g$ can also capture especially challenging shapes, such as densities with flat regions, fat tails, and/or discontinuities. Finally, multivariate prior estimation is also straightforward under our neural-$g$ framework. 

We theoretically justify neural-$g$ by establishing a new universal approximation theorem for neural networks with a softmax output layer. We show that neural-$g$ is capable of approximating \emph{any} valid PMF defined over $\Theta_m$. While there is a very rich literature on the approximation capabilities of neural networks \citep{hornik1989multilayer, hornik1991approximation, lu2017expressive}, existing results mostly concern the approximation of \textit{continuous} functions (i.e. networks with the identity function as the activation in the output layer). Our theory differs in the sense that: a) we establish the approximation capability of neural networks for estimating \textit{PMF}s where the target estimand is both discrete \emph{and} constrained to sum to one, and b) our theory is established for the \emph{softmax} activation function in the output layer rather than the identity function. 

The rest of the paper is structured as follows. Section \ref{sec:neural-g} introduces the neural-$g$ estimator for univariate mixing density estimation and establishes its universal approximation capability. Section \ref{sec:implement} introduces a weighted average gradient method for numerically optimizing neural-$g$ which accelerates the convergence of stochastic gradient descent (SGD). Section \ref{sec:simulate} demonstrates neural-$g$'s flexibility over NPMLE and Efron's $g$ through several simulated examples. Section \ref{sec:multivariate} extends neural-$g$ to multivariate prior estimation. Section \ref{sec:real} applies neural-$g$ to real data applications, including Poisson mixture models and a measurement error model. Section \ref{sec:conclusion} concludes the paper with a brief discussion.

\section{The Neural-$g$ Estimator} \label{sec:neural-g}

\subsection{Motivating Examples}\label{sec:motiv}

Before formally introducing our proposed neural-$g$ estimator, we first illustrate its flexibility in two simple examples. Under the mixture model \eqref{eq:mixturemodel}, we generated a synthetic dataset $y_1, \ldots, y_n$ of size $n=4000$ as follows:
\begin{itemize}
    \item \textbf{Point mass prior}: $y \mid \theta \sim \mathcal{N}(\theta, 0.5),~ \theta = -5 \text{ w.p. } 0.3,  \theta = 0 \text{ w.p. } 0.4, \theta = 5 \text{ w.p. } 0.3.$, where ``w.p.'' denotes ``with probability.''
    \item \textbf{Gaussian prior}: $y \mid \theta \sim \mathcal{N}(\theta, 1),~ \theta \sim \mathcal{N}(0,1).$
    \end{itemize}
These are Gaussian location mixtures where the mixing density $\pi$ is either a three-component discrete mixture or a standard normal density. We compared the estimated neural-$g$ under default hyperparameters (described in Section \ref{sec:implement}) to the NPMLE and Efron's $g$ with four different pairs of smoothing parameters $(p, \lambda) \in \{ (5,1), (5, 10), (20,1), (20, 10) \}$. For all estimators, we used an equispaced grid of $m=100$ points from $y_{\min}$ to $y_{\max}$ for $\Theta_m$.

The performance of these different density estimators is summarized in Figure \ref{fig:m1-m2}. The top two panels of Figure \ref{fig:m1-m2} plot the estimated densities, and the bottom two panels plot the estimated cumulative distribution functions (CDFs). For the point mass prior (left two panels), the neural-$g$ estimator is shown to be very close to the true discrete mixture of three atoms, and neural-$g$'s estimated CDF is very close to the true step function CDF. For the Gaussian prior (right two panels), neural-$g$ also effectively recovers the Gaussian shape, and its estimated CDF is the closest to the true continuous CDF among all estimators. 

\begin{figure}[t]
    \centering
    \includegraphics[width=1.0\textwidth]{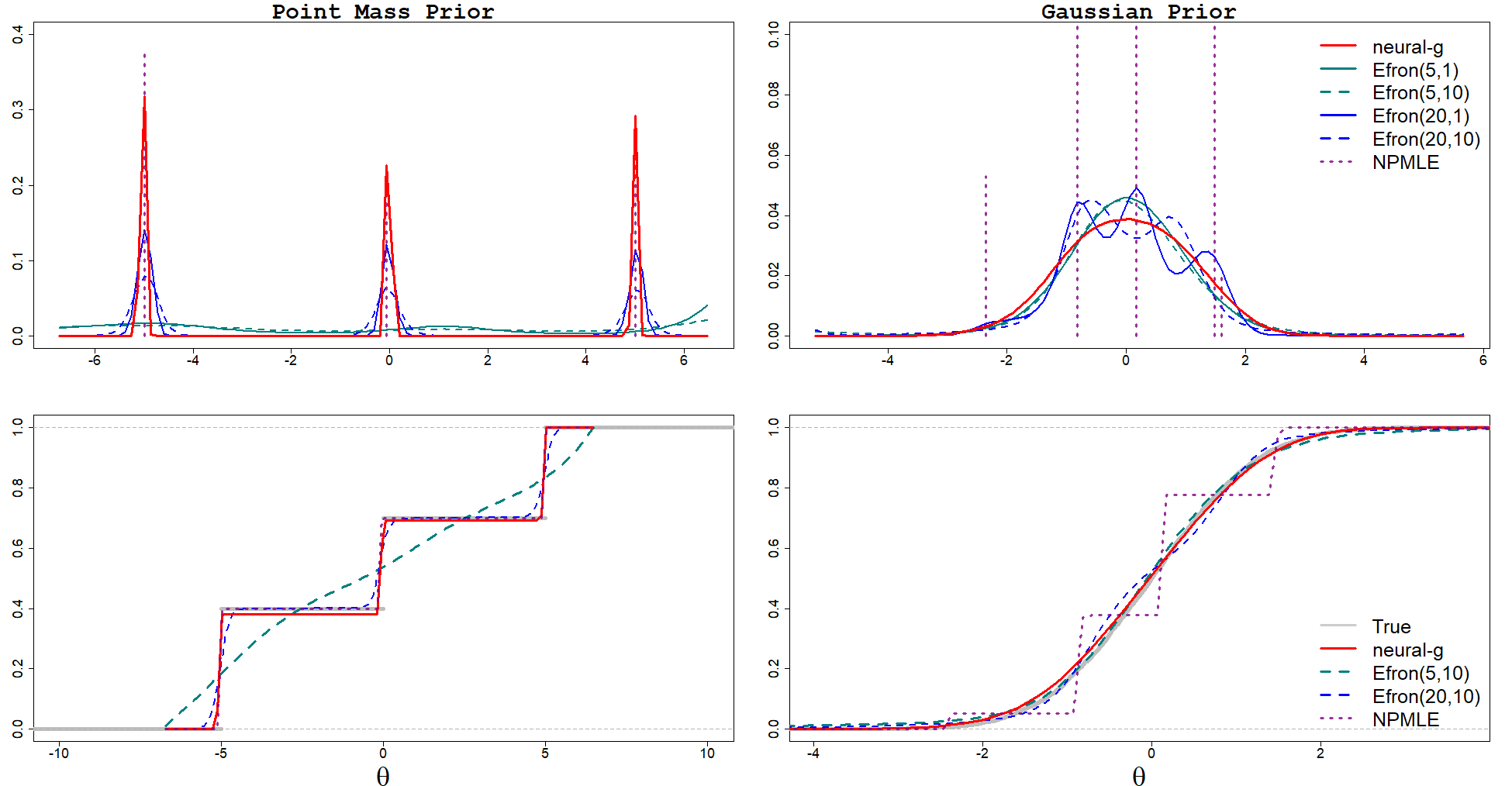}
\caption{Performance of different estimators when the true mixing density is a point mass prior (left two panels) or a Gaussian prior (right two panels). The top two panels plot the estimated densities, while the bottom two panels plot the estimated CDFs.}\label{fig:m1-m2}
\end{figure}

In comparison to neural-$g$, Figure \ref{fig:m1-m2} shows that NPMLE performs very well for the point mass prior when there are only a few atoms. However, NPMLE is too coarse to capture the Gaussian prior; in this case, NPMLE only estimated five atoms. Efron($5, 1$) and Efron($5, 10$) do not perform well for the point mass prior, resulting in estimates that are overly smooth. However, Efron($20, 1$) and Efron($20, 10$) are able to approach the true point mass mixture. Meanwhile, for the Gaussian prior, Efron($5, 1$) and Efron($5, 10$) perform well, but Efron ($20, 1$) and Efron($20, 10$) result in multimodal density estimates that are \emph{too} wiggly. These results show that the performance of Efron's $g$ is heavily influenced by the degrees of freedom $p$ and the penalty parameter $\lambda$. In contrast, neural-$g$ performed well in both examples under the exact same network architecture.

These two simple Gaussian location mixture examples show that neural-$g$ can adapt to \emph{both} very discrete (i.e. non-smooth) and very smooth mixing densities, while requiring no extra tuning of the network hyperparameters. In Section \ref{sec:simulate}, we further showcase mixing densities (i.e. uniform, piecewise constant, heavy-tailed, and bounded) that are difficult for either NPMLE or Efron's $g$ to estimate well but which neural-$g$ estimates with ease.

\subsection{Introducing Neural-$g$} \label{sec:net-npmle}

In the absence of \textit{a priori} knowledge about the characteristics of the true mixing density, we would like to be able to estimate both non-smooth \emph{and} smooth priors $\pi$ in \eqref{eq:mixturemodel}. Ideally, we would also like to avoid needing to tune smoothness parameters. This motivates us to consider neural networks for estimation of $\pi$. We consider \textit{multilayer perceptrons} (MLPs), or fully connected feedforward networks constructed by composing activated linear transformations \citep{GoodfellowDeep2016}. 

We first briefly review the basic MLP structure. For $k=1,\dots,K$, let $g_k$ denote the feedforward mapping represented by $h_k$ hidden nodes, where $g_k: \mathbb{R}^{h_k}\mapsto \mathbb{R}^{h_{k+1}}$ is defined as $g_k(\bm{z})= \sigma_k(\bm{W}^{(k)}\bm{z} +\bm{b}^{(k)})\in\mathbb{R}^{h_{k+1}}$, $\bm{z} \in\mathbb{R}^{h_k}$ is the input variable of $g_k$, $\bm{W}^{(k)} \in \mathbb{R}^{h_{k+1} \times h_k}$ is a weight matrix, $\bm{b}^{(k)} \in \mathbb{R}^{h_{k+1}}$ is a bias vector, and $\sigma_k(\cdot)$ is a (possibly nonlinear) activation function applied elementwise to $\bm{W}^{(k)}\bm{z} +\bm{b}^{(k)}$. Letting $\phi = ( \bm{W}^{(1)}, \ldots, \bm{W}^{(K)}, \bm{b}^{(1)}, \ldots, \bm{b}^{(K)})$ denote the network parameters, a $K$-layer MLP function $g_{\phi}: \mathbb{R}^{h_1} \mapsto \mathbb{R}^{D}$ can be defined by the composition of the functions,
\begin{equation*}
    g_{\phi}(\bm{x}) =  g_{K} \circ \dots \circ g_{1}(\bm{x}),
\end{equation*} 
where $g_K: \mathbb{R}^{h_{K}} \rightarrow \mathbb{R}^{D}$ maps the final $g_{K-1} \circ \dots \circ g_1(\bm{x})$ to the $D$-dimensional output space of $g_{\bm{\phi}}$. The first $K-1$ layers are hidden layers, while the final layer (i.e., layer $K$) of the MLP is the \textit{output} layer. In practice, it is common for the first $K-1$ activations $\sigma_1, \ldots, \sigma_{K-1}$ to be taken as the rectified linear unit (ReLU) activation \citep{nair2010rectified}, i.e. $\sigma_1(x) = \cdots = \sigma_{K-1}(x) = \max\{ 0, x \}$. Meanwhile, the activation function in the output layer $\sigma_K$ depends on the final $D$ targets. 

Like NPMLE and Efron's $g$, we assume that the prior $\pi$ in the mixture model \eqref{eq:mixturemodel} has a finite support with $m$ points $\Theta_m =  \{ \theta_1, \ldots, \theta_m \}$. Given $\Theta_m$, we define a $K$-layer MLP family $\mathcal{G}_m$ as 
\begin{equation} \label{eq:MLP-family}
    \left\{ g_{\phi} : \mathbb{R}^{1} \mapsto [0,1]^{1}~~\bigg|~~0 \leq g_{\phi}(\theta_j) \leq 1 \text{ for all } j = 1, \ldots, m, \text{ and } \sum_{j=1}^{m} g_{\phi}( \theta_j) = 1 \right\},
\end{equation}
where $\phi$ is the set of weight and bias parameters, and $g_{\phi}$ is constructed as a single-input single-output (SISO) function where $g_{\phi}(\theta_j) = \mathbb{P}(\theta = \theta_j)$ represents the MLP output for the $j$th support point $\theta_j \in \Theta_m$. 
Denote the output of the $K$-layer MLP as $g_{\phi}(\Tm) = (g_{\phi}(\theta_1),\ldots,g_{\phi}(\theta_m))^\top \in \mathbb{R}^{m}$ where the input is $m \times 1$ vector $\Theta_m = (\theta_1, \ldots, \theta_m)^\top$.
In order to estimate a valid density on $\Theta_m$, it must necessarily be the case that all $g_{\phi}(\theta_j)$'s lie in $[0,1]$ and sum to one. To enforce this constraint, the activation function in the output layer is specified as the \textit{softmax} transformation. That is, for $\bm{z} = g_{K-1} \circ \cdots \circ g_1(\Theta_m)$, the mapping in the output layer $g_K (\bm{z}) : \mathbb{R}^{h_{K-1}} \mapsto \mathbb{R}^{m}$ is given by
\begin{align} \label{eq:output-layer}
    g_K (\bm{z}) = \sigma_{\text{softmax}} ( \bm{W}^{(K)} \bm{z} + \bm{b}^{(K)} ),~~ \bm{W}^{(K)} \in \mathbb{R}^{m \times h_{K-1}},~ \bm{b}^{(K)} \in \mathbb{R}^{m},   
\end{align}
where the $j$th element of $\sigma_{\text{softmax}}(\bm{x})$ for $\bm{x} = (x_1, \ldots, x_m)^\top \in \mathbb{R}^{m}$ is defined as 
\begin{align} \label{eq:softmax}
    \sigma_{\text{softmax}}(\bm{x})_j = \frac{e^{x_j}}{\sum_{j=1}^{m} e^{x_j}},~~~j = 1, \ldots, m.
\end{align}
For a given finite set $\Theta_m$, the family of MLP-based PMFs $\mathcal{G}_m$ in \eqref{eq:MLP-family} is broader than the parametric exponential family \eqref{eq:efron-npmle} of \cite{efron2016empirical}. However, we point out that there is a continuous analogue for the family \eqref{eq:efron-npmle} \citep{efron2016empirical, Efron2014modeling}, whereas this is not the case for \eqref{eq:MLP-family}. Even so, \cite{Efron2014modeling} and \cite{efron2016empirical}  have noted that all practical implementations of Efron's $g$ require discretization, whereas continuous constructions are purely theoretical. Technically speaking, all probabilities $g_{\phi} (\theta_j) = P(\theta = \theta_j)$ are strictly positive under the softmax transformation \eqref{eq:softmax}. However, in practice, many of outputted individual probabilities $g_{\phi}(\theta_j)$ can be extremely close to, or effectively, zero. In addition, the PMFs estimated by neural-$g$ can be quite smooth in practice (just like the PMFs estimated by Efron's $g$), provided that the grid size $m$ is sufficiently large.

Having defined the neural network family $\mathcal{G}_m$ in \eqref{eq:MLP-family}, we now introduce the neural-$g$ estimator for estimating the prior $\pi$ in \eqref{eq:mixturemodel}. Let $\Phi$ denote the parameter space for network parameters $\phi$ when $g_{\phi} \in \mathcal{G}_m$. Neural-$g$ solves the optimization,
\begin{equation}\label{eq:net-npmle-phi}
    \hat{\phi}  = \underset{\phi \in \Phi }{\mathrm{argmin}} -\frac{1}{n} \sum_{i=1}^n\mathrm{log}\bigg\{ \sum_{j=1}^{m} f(y_i \mid \theta_j) g_{\phi}(\theta_j) \bigg\}.
\end{equation}
After optimizing the neural network parameters $\phi$ according to \eqref{eq:net-npmle-phi}, the neural-$g$ estimator is given by 
$\widehat{\pi}_{\text{neural-}g} = \sum_{j=1}^{m} g_{\widehat{\phi}}(\theta_j) \delta_{\theta_j}$, where $g_{\widehat{\phi}}(\theta_j) =  \widehat{\mathbb{P}}(\theta = \theta_j), j = 1, \ldots, m$.
  

\subsection{Neural-$g$'s Approximation Capability}\label{sec:approx-net-npmle}

In the past several decades, a number of universal approximation theorems have established that MLPs can approximate \emph{continuous}, bounded functions arbitrarily well, provided that the MLPs are wide enough \citep{hornik1989multilayer, hornik1991approximation, lu2017expressive}. These MLPs use the identity function as the activation in the output layer. Although MLPs with \emph{softmax} output layers are also widely used -- especially for classification tasks \citep{GoodfellowDeep2016}, their theoretical properties are less well-understood. To our knowledge, the only theoretical result for MLPs with softmax output layers is given in \cite{asadi2020approximation}. However, the result in \cite{asadi2020approximation} is \emph{not} applicable in our case, since \cite{asadi2020approximation} require the neural network inputs to be \emph{continuous} and lie in the unit hypercube $[0,1]^d$. In contrast, the inputs for neural-$g$ are \emph{discrete} points $\theta_m \in \Theta_m$, where $\Theta_m$ is a \textit{finite} set. Thus, new theory is needed 
to justify neural-$g$.

In the next theorem, we establish a new universal approximation theorem regarding neural-$g$'s capability to learn arbitrary PMFs. For a vector $\bm{v} = (v_1, \ldots, v_m)^\top \in \mathbb{R}^{m}$, we denote its infinity norm as $\lVert \bm{v} \rVert_{\infty} = \max_{1 \leq j \leq m} | v_j |$. 

\begin{theorem}\label{theo:net-npmle}
	Let $p(\theta)$ be an arbitrary PMF defined on $\Tm = \{ \theta_1, \ldots, \theta_m \}$, and let $p(\Theta_m) = ( p(\theta_1), \ldots, p(\theta_m))^\top$ denote the probability vector where $p(\theta_j) = \mathbb{P}(\theta = \theta_j), j = 1, \ldots, m$. Meanwhile, let $g_{\phi}(\Theta_m) = ( g_{\phi}(\theta_1), \ldots, g_{\phi}(\theta_m))^\top$ denote the probability vector where $g_{\phi} \in \mathcal{G}_m$ and $\mathcal{G}_m$ is the MLP family defined in \eqref{eq:MLP-family}. Then for any arbitrarily small $\epsilon>0$, there exists a single-hidden-layer network $g_{\phi} \in \mathcal{G}_m$ such that
\begin{equation}\label{eq:appro-net-npmle}
\lVert g_{\phi}(\Tm) - p(\Tm) \rVert_{\infty} < \epsilon.
\end{equation}
\end{theorem}

\begin{proof}
    Appendix \ref{sec:theo-net-npmle}.
\end{proof}
\noindent Theorem \ref{theo:net-npmle} shows that there exists a one-hidden-layer MLP $g_{\phi} \in \mathcal{G}_m$ which can approximate \emph{any} PMF function on $\Tm$ arbitrarily well. This justifies restricting attention to the family $\mathcal{G}_m$ in \eqref{eq:MLP-family} for optimizing the mixing density $\pi$ in \eqref{eq:mixturemodel}.

Next, we show that that there also exists a PMF $\pi_{\text{neural-}g} = \sum_{j=1}^{m} g_{\phi}(\theta_j) \delta_{\theta_j}$ which can approximate the solution to the loss function,
\begin{equation} \label{eq:optim-over-Thetam}
    \argmin_{\pi \in \mathcal{F}({\Tm})}-\frac{1}{n} \sum_{i=1}^{n} \log \left\{ \sum_{j=1}^{m} f(y_i \mid \theta_j) \pi(\theta_j) \right\}, 
\end{equation}
where $\mathcal{F}({\Theta_m})$ is a particular family of PMFs on $\Theta_m$. For example, if $\mathcal{F}(\Tm)$ is the set of all PMFs $\sum_{j=1}^{m} w_j \delta_{\theta_j}$ where $\bm{w} = (w_1, \ldots, w_m)^\top \in \mathcal{S}_{m-1}$, then \eqref{eq:optim-over-Thetam} is the NPMLE solution \eqref{eq:kwnpmle_sol}. Meanwhile, if $\mathcal{F}(\Tm)$ is the set of all PMFs $\sum_{j=1}^{m} \pi_j(\bm{\alpha}) \delta_{\theta_j}$ where $\pi(\bm{\alpha})$ is the curved exponential family in \eqref{eq:efron-npmle}, then \eqref{eq:optim-over-Thetam} is a (nonpenalized) Efron's $g$ estimator.
We first state the following assumptions:
\begin{enumerate}[label=(A\arabic*)]
\item There exists at least one $\theta \in \Theta_m$ so that $f(y_i \mid \theta) > 0$ for all $i = 1, \ldots, n$. Furthermore, $f(y_i \mid \theta_j ) < \infty$ for all $i = 1, \ldots, n$ and all $\theta_j \in \Theta_m$.
\item A global minimum $\widehat{\pi}_m$ for \eqref{eq:optim-over-Thetam} exists.
\end{enumerate}
Assumption (A1) ensures that the loss function \eqref{eq:optim-over-Thetam} is well-defined given observed data $\bm{y}$, while Assumption (A2) assumes that \eqref{eq:optim-over-Thetam} can be solved over the class of PMFs $\mathcal{F}(\Theta_m)$. The next theorem states that there exists a density $\pi_{\text{neural-}g}$ which can approximate the solution to \eqref{eq:optim-over-Thetam} well.




\begin{theorem}\label{theo:net-npmle-2} 
	For $\Theta_m = \{ \theta_1, \ldots, \theta_m \}$, suppose that Assumptions (A1)-(A2) hold. Let $\ell(\pi)$ denote the loss function in \eqref{eq:optim-over-Thetam}, and let $\widehat{\pi}_m$ be the global minimum of \eqref{eq:optim-over-Thetam}. Then there exists a PMF $\pi_{\text{neural-}g} = \sum_{j=1}^{m} g_{\phi}(\theta_j) \delta_{\theta_j}$, where $g_{\phi} \in \mathcal{G}_m$ is a single-hidden-layer MLP in the family \eqref{eq:MLP-family}, such that for any arbitrarily small $\epsilon>0$, 
\begin{equation}
    \vert \ell(\pi_{\text{neural-}g}) - \ell(\widehat{\pi}_m) \vert <\epsilon.
\end{equation}
\end{theorem}
\begin{proof}
    Appendix \ref{sec:theo-net-npmle-2}.
\end{proof}
\noindent Theorem \ref{theo:net-npmle-2} implies that neural-$g$ can approximate both the NPMLE \emph{and} smooth estimators such as Efron's $g$.

    It is important to stress that Theorem \ref{theo:net-npmle-2} does \textit{not} imply that the estimator  $\widehat{\pi}_{\text{neural-}g} = \sum_{j=1}^{m} g_{\widehat{\phi}}(\theta_j) \delta_{\theta_j}$, where $\widehat{\phi}$ is the global minimum of \eqref{eq:net-npmle-phi}, is the same as the global minimum of \eqref{eq:optim-over-Thetam} for any arbitrary PMF family $\mathcal{F}(\Theta_m)$. This will not be the case in general, and further, since \eqref{eq:net-npmle-phi} is a nonconvex optimization problem, our method only finds a local minimum for $\widehat{\pi}_{\text{neural-}g}$. Instead, Theorem \ref{theo:net-npmle-2} states that neural-$g$ has the \emph{capability} to approximate either atomic solutions (e.g. NPMLE) \emph{or} smooth solutions (e.g. Efron's $g$) for the mixing density $\pi$ in \eqref{eq:optim-over-Thetam}. This demonstrates neural-$g$'s adaptability to different prior shapes. 


\subsection{Related Work}

We pause briefly to contrast neural-$g$ with other existing works on neural networks for mixing density estimation. \citet{Vandegar2021}, \citet{Dorkhorn2020}, \citet{QiuWang2021}, and \citet{Wang2024GBNPMLE} have proposed to estimate the prior $\pi$ in the latent mixture model \eqref{eq:mixturemodel} using deep \textit{generative} models. These methods do not estimate $\pi$ directly but instead use DNNs to learn a mapping $G(\cdot)$ from a noise vector to the target density $\pi$. The estimated map is then used to approximately \emph{generate} samples from the prior $\pi$ in \eqref{eq:mixturemodel}. 

Neural-$g$ is not a generative model. Instead, neural-$g$ directly estimates a PMF over a discrete parameter space $\Theta_m$ for $\pi$ in \eqref{eq:mixturemodel}. Whereas the aforementioned deep generative approaches may be very well-suited for estimating \textit{continuous} priors, neural-$g$ has greater flexibility to estimate \emph{non}-continuous priors such as atomic and piecewise continuous priors. Neural-$g$ can also estimate very smooth, continuous priors by setting the grid size $m$ for $\Theta_m$ in \eqref{eq:MLP-family} to be sufficiently large.

\section{Implementation of Neural-$g$}\label{sec:implement}

In this section, we provide details for implementing neural-$g$ given support points $\Theta_m = \{ \theta_1, \ldots, \theta_m \}$. We first describe our default neural network architecture for neural-$g$. We then introduce a variant of the SGD algorithm for optimizing \eqref{eq:net-npmle-phi}.




Theorem \ref{theo:net-npmle} suggests that a network $g_{\phi} \in \mathcal{G}_m$ with a single hidden layer can learn any PMF on $\Theta_m$ arbitrarily well. However, in practice, the number of hidden neurons in a shallow network may need to be unfeasibly large in order to learn the target PMF well \citep{GoodfellowDeep2016, Mhaskar2017}. \textit{Deep neural networks} (DNNs), or networks with at least two hidden layers, often have better empirical performance than shallow networks and typically require much fewer network parameters to achieve comparable accuracy \citep{zagoruyko2016wide, he2016deep, Mhaskar2017}. 

To allow neural-$g$ to capture more complex patterns in the data with lower sample complexity, we set our default $g_{\phi}$ in \eqref{eq:MLP-family} to consist of $L=4$ hidden layers (i.e. total depth $K=5$) and $h=500$ neurons per hidden layer. However, in our sensitivity analysis in Appendix \ref{sec:sensitivity}, we found that the performance of neural-$g$ was fairly robust to different choices of $(L, h)$, with only modestly worse performance when $L=1$. For the hidden layers of $g_{\phi}$, we use ReLU activation functions $\sigma(x)=\max(0, x)$. ReLU is much more computationally efficient and less likely to lead to numerical problems like vanishing gradients than other activation functions such as sigmoid or hyperbolic tangent \citep{nair2010rectified}.

SGD is the standard approach for optimizing DNNs \citep{robbins1951stochastic, EmmertStreib2020FrontiersinAI}. Under SGD, a minibatch of size $S \ll n$ is randomly sampled from the training data, and the average of the $S$ individual gradients is used to approximate the exact gradient in the gradient descent algorithm \citep{Curry1944TheMO}. Typically, the full data is randomly divided into $B = \lceil n/S \rceil$ disjoint minibatches. SGD completes a full pass through all $B$ batches (or an \textit{epoch}) by updating the network parameters using one of the batches in each iteration of SGD. Training typically terminates after $E$ epochs have completed, where $E$ is sufficiently large.



One of the issues with SGD is that convergence can be slow. This is because SGD always moves in approximately the direction of steepest descent, which is not ideal if the loss surface is actually very flat. In this case, steep gradients slow down the movement towards a local minimum. The NPMLE loss function \eqref{eq:kwnpmle_sol} often exhibits a flat likelihood surface with multiple local maxima \citep{WangJRSSB2007}, and the same is true for the neural-$g$ loss function \eqref{eq:net-npmle-phi} (where the weights $\bm{w}$ in \eqref{eq:kwnpmle_sol} are replaced with $g_{\phi}(\theta_j)$). This problem is further compounded by the use of the softmax function in the output layer of $g_{\phi}$ for neural-$g$. It has been rigorously proven that optimizing log-probabilities under softmax transformation must exhibit slow convergence for gradient descent-based algorithms due to long suboptimal plateaus \citep{MeiNeurIPS2020}.

To accelerate the convergence of DNN training for neural-$g$, we propose a weighted average gradient (WAG) approach. Let the randomly shuffled data $\bm{y}$ be denoted as $\bm{y}^* = \{ \bm{y}_1^*,\ldots, \bm{y}_B^* \}$ which is partitioned into $B = \lceil n/s \rceil$ minibatches.
In each $t$th iteration of WAG, the update for $\phi$ is given by
\begin{equation}\label{eq:avg-grad}
    \phi^{(t)} \leftarrow \phi^{(t-1)} - {\eta^{(t)}}\bigg(wS^{-1}\nabla\ell(\bm{y}^*_b, \phi^{(t-1)})+ (1-w)S^{-1}(t-2)^{-1}\textstyle{\sum_{k=1}^{t-2}} \nabla\ell(\bm{y}_{b^{(k)}}^*, \phi^{(k)}) \bigg),
\end{equation}
where $\eta^{(t)}$ is an adaptively tuned step size, and $0\leq w \leq 1$ represents a weight parameter that assigns weight $w$ to the current batch gradient $S^{-1} \nabla\ell(\bm{y}^*_b, \phi^{(t-1)}) = S^{-1} \textstyle{\sum_{i \in \bm{y}_b^*}}\nabla\ell(y_i, \phi^{(t-1)})$ and weight $(1-w)$ to the total average of past gradients from \emph{all} previous batches and epochs $S^{-1}(t-2)^{-1}\textstyle{\sum_{k=1}^{t-2}} \nabla\ell(\bm{y}_{b^{(k)}}^*, \phi^{(k)})$, where $\bm{y}_{b^{(k)}}^*$ denotes the batch used to update $\phi$ in iteration $k$. 

We found that for neural-$g$, WAG converged faster than alternative optimizers such as SGD with momentum \citep{qian1999momentum} or the Adam optimizer \citep{kingma2014adam}. These other optimizers also employ a weighted average but replace the total past average in \eqref{eq:avg-grad} with the most recent iterate $\phi^{(t-1)}$. By averaging over both previous batch gradients and previous epochs in each update of $\phi$, WAG gives even greater weight to all past gradients than momentum or Adam. This stabilizes neural-$g$ training by making neural-$g$ even more predisposed towards moving in the same direction as past gradients.
 
In all of our synthetic and real data examples, we set the weight parameter $w$ to be 0.6 in \eqref{eq:avg-grad} by default and the maximum number of training epochs $E$ to be $8000$. The adaptive step size was set as $\eta^{(t)} = {\eta}{t^{-0.2}}$ where $\eta = 0.0003$. To avoid needing to run WAG for all $E$ epochs, we adopted a commonly used stopping rule where we compared the loss value at the $t$th iteration $\ell^{(t)} = - n^{-1} \sum_{i=1}^{n} \log \{ \sum_{j=1}^{m} f(y_i \mid \theta_j) g_{\phi^{(t)}} (\theta_j) \}$ to the loss value from $c < t$ iterations ago.  Training terminates if $\vert \ell^{(t)}-\ell^{(t-c)}\vert < \epsilon$ for some small $\epsilon > 0$, i.e. there has been no significant improvement in the loss \eqref{eq:net-npmle-phi} after $c$ iterations. For neural-$g$, we set default values of $\epsilon = 0.01$ and $c=10$. 

\RestyleAlgo{ruled}
\begin{algorithm}[t!]
\SetKwInput{KwData}{Input}
\SetKwInput{KwResult}{Output}
\caption{{Training neural-$g$}}\label{alg:net-npmle}
\vspace{0.1cm}
\KwData{\small{Data $\bm{y} = (y_1, \ldots, y_n)$, discrete support $\Theta_m = \{ \theta_1, \ldots, \theta_m \}$}, maximum number of epochs $E$, batch size $S$, weight $w$, stopping criteria  $\epsilon$ and $c$}
\vspace{0.1cm}
\KwResult{\small{Network parameters $\phi$ for neural-$g$}}
\vspace{0.3cm}

$B = \lceil n/S \rceil$

$t = 0$
\vspace{.3cm}

\textbf{Randomly initialize} $\phi^{(0)}$
\vspace{.3cm}

    \For{$e$ in $1,\ldots,E$}{
    \vspace{0.1cm}
    \renewcommand{\labelitemi}{$\vcenter{\hbox{\small$\bullet$}}$}
    \begin{itemize}[leftmargin=*]
        \item     \textbf{\small{Randomly shuffle}} dataset {\small $\bm{y}$} as {\small $\bm{y}^*$} and partition {\small $\bm{y}^*$} into {\small $\{\bm{y}_1^*,\ldots, \bm{y}_B^* \}$}
    \end{itemize}
    \vspace{-0.2cm}
    \For{$b$ in $1, \ldots, B$}{
    \vspace{-.2cm}
    
    \begin{itemize}[leftmargin=*]
    \item \textbf{\small{Iterate}} $t= t + 1$
    \end{itemize}    
    \vspace{-.8cm}
    
    \begin{itemize}[leftmargin=*]
        \item         \textbf{\small{Evaluate}} approximate gradient of neural-$g$ loss \eqref{eq:net-npmle-phi} as
    \end{itemize}
    \vspace{-0.1cm}
    \begin{equation*}
    S^{-1} \nabla \ell(\bm{y}^*_b,\phi^{(t-1)}) = S^{-1} \nabla_{\phi^{(t-1)}} \big\{ \textstyle \sum_{i \in \bm{y}^*_b}\mathrm{log} \big[ \textstyle \sum_{j=1}^m f(y_i \mid \theta_j) g_{\phi^{(t-1)}}(\theta_j) \big] \big\}
    \end{equation*}
    \begin{itemize}[leftmargin=*]
    \item         \textbf{\small{Update}} network parameters $\phi$ according to gradient averaging \eqref{eq:avg-grad} as
    \end{itemize}
    \vspace{0.1cm}
    \begin{equation*}
    \phi^{(t)} \leftarrow \phi^{(t-1)}-  \eta^{(t)} \left[ wS^{-1}\nabla\ell(\bm{y}^*_b, \phi^{(t-1)})+ (1-w)S^{-1}(t-2)^{-1}\textstyle{\sum_{k=1}^{t-2}} \nabla\ell(\bm{y}_{b^{(k)}}^*, \phi^{(k)}) \right]
    \end{equation*}
    \begin{itemize}[leftmargin=*]
    \item \textbf{Evaluate} loss $\ell^{(t)}$ as
    \begin{equation*}
        \ell^{(t)} = - \frac{1}{n} \sum_{i=1}^{n} \log \left\{ \sum_{j=1}^{m} f(y_i \mid \theta_j) g_{\phi^{(t)}} (\theta_j) \right\}
    \end{equation*}
    \end{itemize}
    \vspace{-.3cm}
    
    \If{$t>c$ \textbf{and} $\vert \ell^{(t)}-\ell^{(t-c)}\vert < \epsilon$}{
    \vspace{0.1cm}
    \textbf{stop}
    }
}
}
\end{algorithm}

A detailed summary of the neural-$g$ training algorithm is provided in Algorithm \ref{alg:net-npmle}. Once the estimated network parameters $\widehat{\phi}$ in \eqref{eq:net-npmle-phi} have been obtained by Algorithm \ref{alg:net-npmle}, the final neural-$g$ density estimator is taken to be $\widehat{\pi}_{\text{neural-}g} = \sum_{j=1}^{m} g_{\widehat{\phi}}(\theta_j) \delta_{\theta_j}$.


\section{Synthetic Illustrations}\label{sec:simulate}

\subsection{Estimation Performance}

As discussed in Section \ref{sec:intro}, many commonly used priors of specific practical interest, such as piecewise constant priors and heavy-tailed priors, are particularly challenging for existing $g$-modeling approaches to capture well. By optimizing over a very flexible family of MLPs \eqref{eq:MLP-family}, neural-$g$ overcomes these limitations. In this section, we demonstrate the estimation capabilities of neural-$g$ through several simulated examples.  Code for implementing neural-$g$ is available at \url{https://github.com/shijiew97/neuralG}.  

In all of our simulations, we simulated $n=4000$ observations $y_1, \ldots, y_n$ from different mixture models \eqref{eq:mixturemodel}. We compared neural-$g$ with NPMLE and Efron's $g$, which were implemented using \textsf{R} packages \texttt{REBayes} \citep{KoenkerGu2017} and \texttt{deconvolveR} \citep{narasimhan2020deconvolver} respectively. For Efron's $g$, we considered degrees of freedom $p \in \{ 5, 20 \}$, denoted as Efron(5) and Efron(20) respectively, and a default choice of $\lambda = 1$ for the penalty parameter as recommended by \cite{efron2016empirical}. To estimate $\pi$, we specified $\Theta_m$ as an equispaced grid of $m=100$ points from $y_{\min}$ and $y_{\max}$. We also compared these estimators to the deconvolution estimator in the \textsf{R} package \texttt{decon} \citep{wang2011deconvolution}, which we denote as WW-decon. 

We considered the following six simulation settings:
\begin{enumerate}[label=\Roman*.]
    \item \textbf{Uniform prior}: $y \mid \theta \sim \mathcal{N}(\theta, 1),~ \theta \sim \mathcal{U}(-2, 2)$;
    \item \textbf{Piecewise constant prior}: $y \mid \theta \sim \mathcal{N}(\theta, 1),~ p(\theta) = 0.4~\mathcal{U}(-2,-1) + 0.2~\mathcal{U}(-1,1) + 0.4~\mathcal{U}(1,2)$;
    \item \textbf{Heavy-tailed prior}: $y \mid \theta \sim \mathcal{N}(\theta, 1),~ \theta \sim \text{Gumbel}(2,1)$; 
    \item \textbf{Bounded prior}: $y \mid \theta \sim \text{Log-normal}(\theta, 0.2),~ \theta \sim \text{Beta}(3,2)$;
 \item \textbf{Point mass prior}: $y \mid \theta \sim \mathcal{N}(\theta, 0.5),~ \theta = -5 \text{ w.p. } 0.3,  \theta = 0 \text{ w.p. } 0.4, \theta = 5 \text{ w.p. } 0.3.$;
    \item \textbf{Gaussian prior}: $y \mid \theta \sim \mathcal{N}(\theta, 1),~ \theta \sim \mathcal{N}(0,1).$
\end{enumerate}
Settings I-III are particularly challenging for NPMLE, Efron's $g$, and WW-decon to capture well. In setting IV, the parameter $\theta$ can theoretically take any real value; however, the true prior $\pi(\theta)$ is bounded on (0,1), and we wish to assess whether different methods are able to detect this boundedness. Finally, settings V-VI are the same as those presented in Section \ref{sec:motiv}. Simulation V is intended to provide a setting that is fair to the NPMLE, while Simulation VI provides a setting where we expect Efron's $g$ (with properly tuned smoothing parameters) to perform well. It is worth stressing again that the $g$-modeling approaches all actually estimate a PMF rather than a true continuous density. However, with sufficiently large $m$, the estimated prior $\widehat{\pi}$ can still be very smooth for neural-$g$ and Efron's $g$, and thus, they can still approximate continuous priors well.

We repeated Simulations I-VI for 20 replications and recorded the following metrics averaged across the 20 replicates. First, we compared the discrepancy between the true CDF and the estimated CDF for the different estimators using the Wasserstein-1 metric. Let $\widehat{\pi}$ denote an estimate of the mixing density $\pi$ in \eqref{eq:mixturemodel}, and let $\mathbb{F}_{\widehat{\pi}}$ and $\mathbb{F}_{\pi}$ denote their corresponding CDFs. The Wasserstein-1 metric is defined as
\begin{equation}\label{eq:wass-1}
    W_1(\pi,\widehat{\pi}) = \int_{\Theta_m} \mid \mathbb{F}_{\pi}(x)-\mathbb{F}_{\widehat{\pi}}(x) \mid dx,
\end{equation}
with a smaller $W_1(\pi, \widehat{\pi})$ indicating better estimation of $\pi$.

\renewcommand{\arraystretch}{1.1}
\begin{table}[t]
\centering
\caption{Results for Simulations I-VI averaged across 20 replications. For the heavy-tailed prior (Simulation IV), results are not shown for NPMLE or WW-decon because the \textsf{R} packages \texttt{REBayes} and \texttt{decon} do not support the log-normal density.}
\resizebox{1.0\columnwidth}{!}{%
\begin{tabular}{c|cc|cc|cc|cc|cc|}
\hline
&\multicolumn{2}{c|}{neural-$g$}&\multicolumn{2}{c|}{NPMLE}&\multicolumn{2}{c|}{Efron(5)}&\multicolumn{2}{c|}{Efron(20)}&\multicolumn{2}{c|}{WW-decon}\\
\hline
Prior $\pi$&$W_1(\pi,\widehat{\pi})$&$\text{MAE}$&$W_1(\pi,\widehat{\pi})$&$\text{MAE}$&$W_1(\pi,\widehat{\pi})$&$\text{MAE}$&$W_1(\pi,\widehat{\pi})$&$\text{MAE}$&$W_1(\pi,\widehat{\pi})$&$\text{MAE}$
\\
\hline
\hline
Uniform & \textbf{0.040} & \textbf{0.014} & 0.342 & 0.044 & 0.164 & 0.113 & 0.095 & 0.036 & 0.144 & 0.089 \\
Piecewise constant & \textbf{0.050} & \textbf{0.032} & 0.425 & 0.364 & 0.160 & 0.097 & 0.364 & 0.336 & 0.108 & 0.045 \\
Heavy-tailed & \textbf{0.062} & \textbf{0.023} & 0.329 & 0.038 & 0.104 & 0.067 & 0.353 & 0.107 & 0.259 & 0.139 \\
Bounded & \textbf{0.024} & \textbf{0.003} & -- & -- & 0.268 & 0.064 & 1.28 & 0.048 & -- & -- \\
Point mass & \textbf{4.51} & 0.02 & \textbf{4.51} & \textbf{5e-4} & 4.71 & 0.400 & 4.54 & 0.03 & 4.62 & 0.250 \\
Gaussian & \textbf{0.047} & \textbf{0.017} & 0.315 & 0.044 & 0.069 & 0.026 & 0.084 & 0.026 & 0.144 & 0.245 \\
\hline
\end{tabular}}
\label{tab:m1-m3}
\end{table}

Since the Bayes estimate, or the posterior expectation $\mathbb{E}_{\pi}(\theta \mid y)$ under prior $\pi$, is typically of central interest in empirical Bayes analysis, we also assessed the quality of the estimated posterior means. To do this, we computed the pointwise posterior expectations $\mathbb{E}_{\widehat{\pi}}(\theta_i \mid y_i)$ under the estimated priors $\widehat{\pi}$ at each observed sample $y_i, i = 1, \ldots, n$, using Bayes rule,
\begin{equation}\label{eq:posterior-mean}
    \widehat{E}_i = \mathbb{E}_{\widehat{\pi}} \big ( \theta_i \mid y_i\big) = \frac{\sum_{j=1}^m \theta_j f(y_i\mid \theta_j)\widehat{\pi}(\theta_j)}{\sum_{j=1}^m f(y_i\mid \theta_j)\widehat{\pi}(\theta_j)},~~ i = 1,\dots, n.
\end{equation}
Given the true prior $\pi$, it is easy to attain the \emph{true} pointwise posterior expectations as $E_{i, \text{true}} = \mathbb{E}_{\pi} (\theta_i \mid y_i) = \int_{\Theta} \theta_i \pi(\theta_i \mid y_i) d \theta_i$. Consequently, we calculated the mean absolute error (MAE) between the $\widehat{E}_i$'s and $E_{i, \text{true}}$'s as
\begin{equation} \label{eq:MAE-1}
\text{MAE} = \frac{1}{n} \sum_{i=1}^n \vert \widehat{E}_i-E_{i, \text{true}}\vert.
\end{equation}
Table \ref{tab:m1-m3} reports our results for the metrics \eqref{eq:wass-1} and \eqref{eq:MAE-1} averaged across 20 replications. Table \ref{tab:m1-m3} shows that the estimated CDF under neural-$g$ was on average the closest to the true CDF. In Simulation V (point mass prior), neural-$g$ tied with the NPMLE for smallest average $W_1(\pi, \widehat{\pi})$. This indicates neural-$g$'s capability to estimate fully discrete priors, in addition to continuous or piecewise continuous densities. The posterior expectations under neural-$g$ also had the smallest average MAE  in all simulation settings except for Simulation V where it performed the second best behind NPMLE. It is not surprising that when the true prior is a discrete mixture with only three atoms, the NPMLE performs the best. However, in all other simulation settings, neural-$g$ vastly outperformed NPMLE, and even in Simulation V, neural-$g$ was competitive with NPMLE. In short, the results in Table \ref{tab:m1-m3} show very encouraging performance for neural-$g$. 

\begin{figure}[t]
    \centering
    \includegraphics[width=1.0\textwidth]{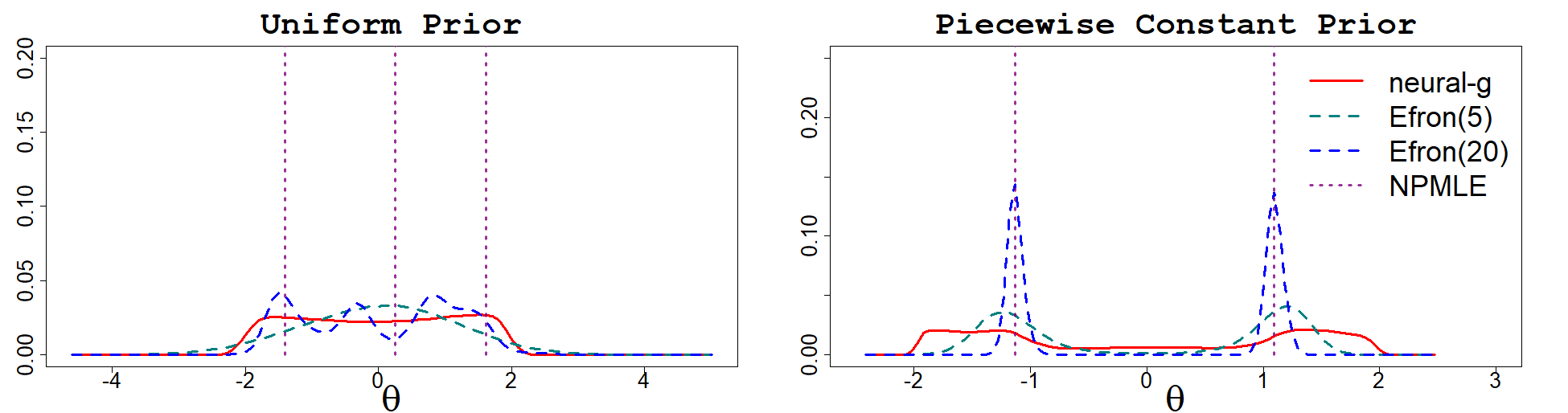}
\caption{\small Left panel: Estimated priors from one replication of Simulation I (Uniform prior). Right panel: Estimated priors from one replication of Simulation II (Piecewise constant prior).}\label{fig:m1-m6}
\end{figure}


Figure \ref{fig:m1-m6} plots the results from one replication of Simulation I (left panel) and one replication of Simulation II (right panel). We observe that neural-$g$ is able to capture the flat regions for both the uniform prior (Simulation I) and the piecewise constant prior (Simulation II). Meanwhile, Efron's $g$ is unable to estimate the flat parts of these prior densities, regardless of the degrees of freedom. In the piecewise constant example (right panel of Figure \ref{fig:m1-m6}), the NPMLE only estimates two atoms, despite there being three nonzero segments over the domain $[-2, 2]$. 

Figure \ref{fig:m3} plots the results from one replication of Simulation III. In this case, the true prior is a Gumbel density which is unimodal, skewed right, and has a fat upper tail. Only neural-$g$ is able to capture all aspects of the true prior. In contrast, NPMLE does not estimate \emph{any} mass in the upper tail. Meanwhile, Efron(5) has a right tail that is too thin, while Efron(20) is multimodal.

\begin{figure}[t]
    \centering
    \includegraphics[width=1.0\textwidth]{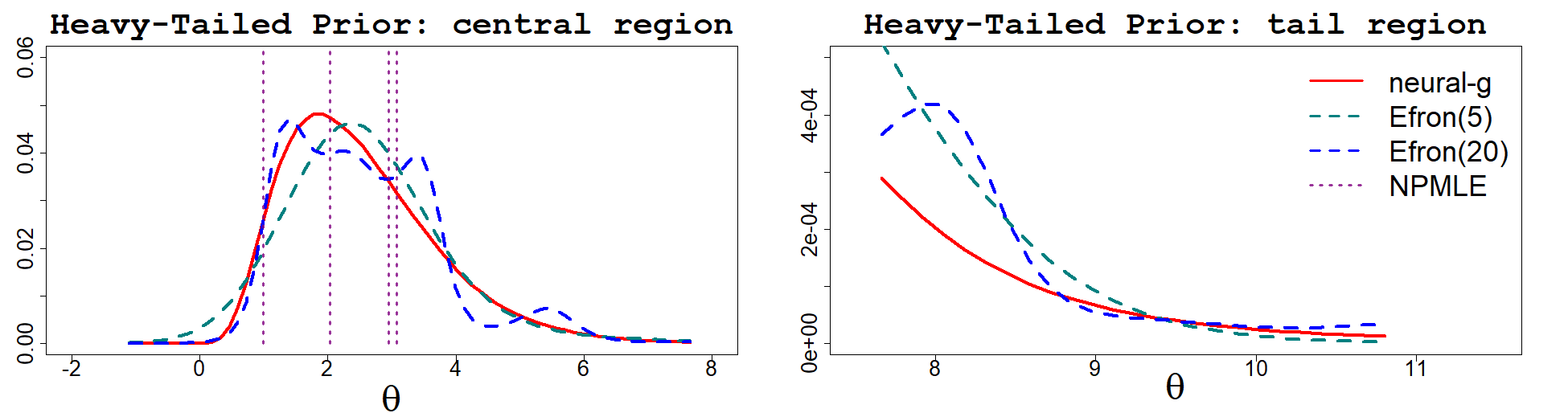}
\caption{\small Estimated priors from one replication of Simulation III (Heavy-tailed prior). Left panel: The estimated densities around the mode. Right panel: The upper tails of the estimated densities.}\label{fig:m3}
\end{figure}

Figure \ref{fig:m4-m5} plots the results from one replication of Simulation IV (left panel) and one replication of Simulation V (right panel). The left panel of Figure \ref{fig:m4-m5} shows that neural-$g$ places practically all of its mass in the correct support $(0,1)$ and correctly infers the left-skewness of the true Beta($3,2$) prior (Simulation IV). In contrast, both Efron(5) and Efron(20) estimate \textit{right}-skewed densities with substantial probability mass when $\theta > 1$, contrary to the ground truth where $\Pi(\theta > 1) =0$. This example showcases neural-$g$'s ability to adapt to a challenging scenario where the true prior actually has a bounded support. Finally, the right panel of Figure \ref{fig:m4-m5} shows unsurprisingly that NPMLE estimates the point mass prior (Simulation V) very accurately. However, neural-$g$ \textit{also} produces a fairly accurate estimate that is very close to the NPMLE, whereas Efron(5) is oversmoothed. 

\begin{figure}[t]
    \centering
    \includegraphics[width=1.0\textwidth]{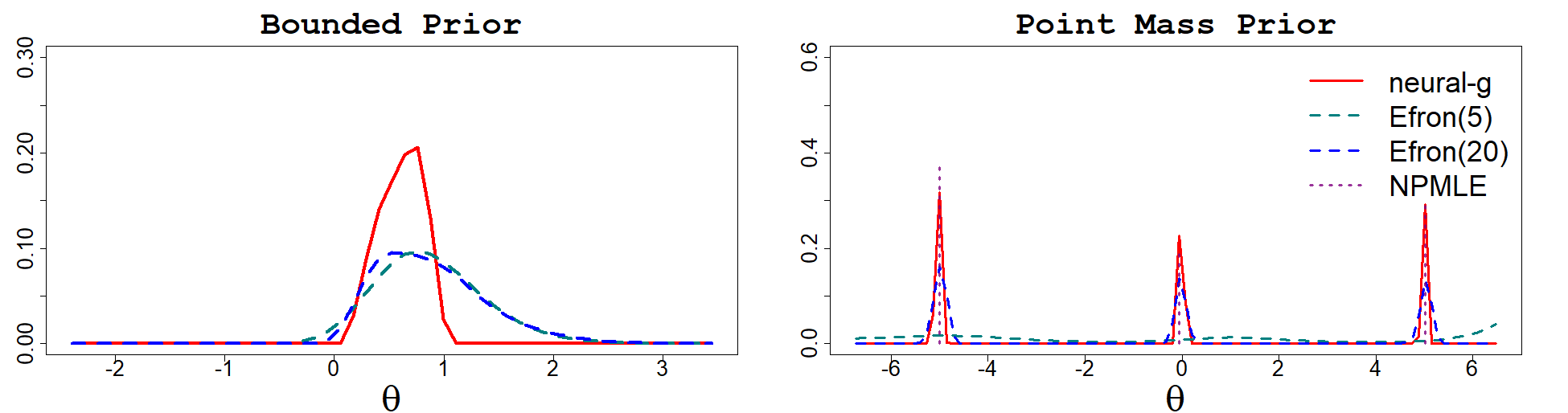}
\caption{Left panel: Estimated priors from one replication of Simulation IV (Bounded prior). Right panel: Estimated priors from one replication of Simulation V (Point mass prior). In the left panel, NPMLE is not plotted because the \texttt{REbayes} package does not support mixture models with a log-normal density.}\label{fig:m4-m5}
\end{figure}

\subsection{Uncertainty Quantification Performance}

\cite{koenker2020empirical} observed that posterior credible intervals constructed under the NPMLE often suffer from undercoverage due to NPMLE's intervals being too narrow. In contrast, \cite{koenker2020empirical} noted that the coverage properties for credible intervals under Efron's $g$ were often satisfactory. To investigate whether neural-$g$ produces acceptable posterior credible intervals, we compared the empirical coverage probabilities (ECPs) of the pointwise 95\% posterior credible intervals under neural-$g$ to those of NPMLE and Efron's $g$ with 20 degrees of freedom. Following \cite{koenker2020empirical}, we first generated $n$ realizations of $\theta_i^{\text{true}} \sim \pi(\theta)$ and $y_i \mid \theta_i^{\text{true}} \sim f(y_i \mid \theta_i^{\text{true}})$, $i = 1, \ldots, n$. We then estimated the posteriors $\widehat{\pi}(\theta_i \mid y_i) = [ \sum_{j=1}^{m} f(y_i \mid \theta_j) \widehat{\pi}(\theta_j) \delta_{\theta_j} ] / [ \sum_{j=1}^{m} f(y_i \mid \theta_j) \widehat{\pi}(\theta_j) ], i = 1, \ldots, n$, and used the 2.5th and 97.5th quantiles of $\widehat{\pi}(\theta_i \mid y_i)$ as the endpoints for our posterior credible intervals. To compute the ECP, we took
\begin{align} \label{ECP}
    \text{ECP} = \frac{1}{n} \sum_{i=1}^{n} \mathbb{I} \big\{ q_{0.025} (\theta_i ) \leq \theta_i^{\text{true}} \leq q_{0.975} (\theta_i ) \big\},
\end{align}
where $q_{0.025}(\theta_i)$ and $q_{0.975}(\theta_i )$ denote the 2.5th and 97.5th quantiles of $\widehat{\pi}(\theta_i \mid y_i)$ respectively.

We considered $n \in \{ 100, 200, 500, 1000 \}$ and repeated this exercise for 100 replicates each of Simulations I-VI. In Table \ref{tab:sim-cov}, we report the average ECPs across these 100 replications, along with the average widths for the 95\% posterior credible intervals in parentheses. We observed the same phenomenon as \cite{koenker2020empirical}, where the average ECPs for the NPMLE were well below the nominal level of 0.95, and the average interval lengths were much narrower than those of neural-$g$ or Efron's $g$. 

Table \ref{tab:sim-cov} shows that as the sample size $n$ increased, neural-$g$ had improved coverage in all simulation settings. Other than Simulation V (Point mass prior), neural-$g$ also had coverage which was close to the nominal level of 0.95 when $n \in \{ 500, 1000 \}$. In Simulation V, Efron(20) had nearly perfect coverage but its credible intervals were quite conservative compared to neural-$g$ and NPMLE. When the true prior is an atomic density with few atoms, we expect the posterior credible intervals to be narrower \citep{koenker2020empirical}, so they may fail to achieve the nominal level of coverage. For Simulation IV (Bounded prior), Efron(20) also had higher coverage than neural-$g$. However, as we saw in Figure \ref{fig:m4-m5}, Efron(20) erroneously estimated the Beta(3,2) prior as being a right-skewed density with substantial mass greater than one. This led to more conservative posterior credible intervals for Efron(20) under Simulation IV. On the other hand, neural-$g$ achieved close to the nominal level when $n \in \{500, 1000\}$ with tighter credible intervals in Simulation IV.



\renewcommand{\arraystretch}{1.5}
\begin{table}[t]
\centering
\caption{Empirical coverage rates for the pointwise 95\% posterior credible intervals in Simulations I-VI, averaged across 100 replicates. The average widths of the credible intervals are reported in parentheses. For the bounded prior (Simulation IV), results are not shown for NPMLE because the \texttt{REbayes} package does not support mixture models with a log-normal density.}
\begin{subtable}[t]{\linewidth}
\centering
\resizebox{1.0\columnwidth}{!}{%
\begin{tabular}{c|ccc|ccc|ccc|}
\hline
&\multicolumn{3}{c|}{{\large Uniform prior}}&\multicolumn{3}{c|}{{\large Piecewise constant prior}}&\multicolumn{3}{c|}{{\large Heavy Tailed prior}}\\
\hline
$n$&neural-g&NPMLE&Efron(20)&neural-g&NPMLE&Efron(20)&neural-g&NPMLE&Efron(20)
\\
\hline
\hline
100&0.905 (3.17)&0.647 (2.24)&\textbf{0.924 (3.29)}&\textbf{0.931 (2.64)}&0.616 (0.810)&0.873 (1.178)&0.891 (2.56)&0.670 (1.86)&\textbf{0.917 (2.69)}
\\
200&0.927 (3.25)&0.705 (2.39)&\textbf{0.933 (3.28)}&\textbf{0.950 (2.68)}&0.686 (0.907)&0.755 (0.992)&\textbf{0.916 (2.63)}&0.723 (2.00)&\textbf{0.916 (2.63)}
\\
500&\textbf{0.939 (3.33)}&0.760 (2.59)&0.935 (3.29)&\textbf{0.958 (2.69)}&0.726 (0.944)&0.588 (0.776)&\textbf{0.937 (2.72)}&0.772 (2.12)&0.919 (2.61)
\\
1000&\textbf{0.945 (3.37)}&0.797 (2.71)&0.938 (3.32)&\textbf{0.969 (2.69)}&0.767 (0.987)&0.476 (0.652)&\textbf{0.939 (2.73)}&0.795 (2.21)&0.920 (2.64)
\\
\hline
\end{tabular}}
\end{subtable}
\begin{subtable}[t]{\linewidth}
\vspace{0.2cm}
\centering
\resizebox{1.0\columnwidth}{!}{%
\begin{tabular}{c|ccc|ccc|ccc|}
\hline
&\multicolumn{3}{c|}{{\large Bounded prior}}&\multicolumn{3}{c|}{{\large Point mass prior}}&\multicolumn{3}{c|}{{\large Gaussian prior}}\\
\hline
$n$&neural-g&NPMLE&Efron(20)&neural-g&NPMLE&Efron(20)&neural-g&NPMLE&Efron(20)
\\
\hline
\hline
100&0.875 (0.498) &-- & \textbf{0.951 (0.675)} &0.644 (0.474)&0.533 (0.286)&\textbf{0.997 (0.976)}&0.891 (2.56)&\textbf{0.917 (2.69)}&0.670 (1.86)
\\
200&0.908 (0.515) &-- & \textbf{0.955 (0.680)} &0.695 (0.462)&0.592 (0.271)&\textbf{0.999 (0.828)}&\textbf{0.916 (2.63)}&0.723 (2.00)&\textbf{0.916 (2.63)}
\\
500&0.927 (0.525) &-- & \textbf{0.959 (0.682)}&0.807 (0.417)&0.763 (0.231)&\textbf{1.000 (0.663)}&\textbf{0.937 (2.72)}&0.772 (2.12)&0.919 (2.61)
\\
1000&0.935 (0.530) &-- & \textbf{0.959 (0.682)}&0.799 (0.416)&0.801 (0.217)&\textbf{1.000 (0.585)}&\textbf{0.939 (2.73)}&0.795 (2.21)&0.920 (2.64)
\\
\hline
\end{tabular}}
\end{subtable}

\label{tab:sim-cov}
\end{table}



\section{Multivariate Neural-$g$}\label{sec:multivariate}

 In Section \ref{sec:neural-g}, we introduced neural-$g$ in the context of univariate mixture models \eqref{eq:mixturemodel}. However, \emph{multivariate} mixture models also arise in many practical applications \citep{WWStatComp2015}. Here, we observe $\bm{Y} = ( \bm{y}_1, \ldots, \bm{y}_n)$, where $\bm{y}_i \in \mathbb{R}^{q}, q \geq 2$, for $i = 1, \ldots, n$, and we assume
\begin{equation} \label{eq:multivariate-mixture}
\bm{y}_i \mid \bm{\theta}_i \sim f( \bm{y}_i \mid \bm{\theta}_i ),~~ \bm{\theta}_i \sim \pi(\bm{\theta}),~~~~ i = 1, \ldots, n,
\end{equation}
where $\pi(\bm{\theta})$ is a $d$-variate density function with $2 \leq d \leq q$, and $f$ is known but $\pi$ is unknown. Several authors have proposed multivariate extensions to NPMLE \citep{WWStatComp2015, feng2018approximate} for estimating the multivariate prior $\pi$ in \eqref{eq:multivariate-mixture}. \cite{WWStatComp2015} noted that extending univariate prior estimation to the multivariate case can be a nontrivial task for NPMLE. In contrast, extension to $d$-variate prior estimation, $d \geq 2$, is very straightforward for neural-$g$.



Similar to the univariate case, multivariate neural-$g$ requires specification of a finite grid of support points $\bm{\Theta}_m = \{ \bm{\theta}_1, \ldots, \bm{\theta}_m \}$, where $\bm{\theta}_j = (\theta_{j1}, \ldots, \theta_{jd})^\top \in \mathbb{R}^{d}, j = 1, \ldots, m$. However, in multi-dimensions, choosing an equispaced grid can become quickly infeasible. For multivariate NPMLE, \cite{feng2018approximate} instead provided a practical two-step algorithm for choosing a set of ``representative'' points for $\bm{\Theta}_m$, and we adopt their proposed procedure to select $\bm{\Theta}_m$. 
Briefly, $\bm{\Theta}_m$ is constructed inductively via a sequence of univariate conditional MLEs given the data $\bm{Y}$; see \cite{feng2018approximate} for a detailed discussion. Given $\bm{\Theta}_m$, we define the family of MLPs $\widetilde{\mathcal{G}}_{m}$ as $\{ g_{\bm{\phi}}: \mathbb{R}^{d} \mapsto [0,1]^1 ~\big|~ 0 \leq g_{\bm{\phi}}(\bm{\theta}_j) \leq 1 \text { for all } j = 1, \ldots, m, \textrm{ and } \sum_{j=1}^{m} g_{\bm{\phi}}(\bm{\theta}_j) = 1 \}$, where $\bm{\phi} \in \bm{\Phi}$ are the network parameters and $g_{\bm{\phi}}(\bm{\theta}_j) = \mathbb{P}(\bm{\theta} = \bm{\theta}_j)$.

The multivariate neural-$g$ estimator is obtained by minimizing the following loss function with respect to $\bm{\phi}$,
\begin{equation}\label{eq:net-npmle-phi-d}
    \widehat{\bm{\phi}}  = \underset{\bm{\phi} \in \Phi}{\mathrm{argmin}} - \frac{1}{n} \sum_{i=1}^n\mathrm{log}\bigg\{ \sum_{j=1}^m f (\bm{y}_i \mid \bm{\theta}_j)  g_{\bm{\phi}}(\bm{\theta}_j) \bigg\}.
\end{equation}
Given $\widehat{\bm{\phi}}$ in \eqref{eq:net-npmle-phi-d}, the multivariate neural-$g$ etimator is then $\widehat{\pi}_{\text{neural-}g} = g_{\widehat{\bm{\phi}}}(\bm{\Theta}_m) = \sum_{j=1}^{m} g_{\widehat{\bm{\phi}}} (\bm{\theta}_j) \delta_{\bm{\theta}_j}$, where $g_{\widehat{\bm{\phi}}} (\bm{\theta}_j) = \widehat{\mathbb{P}} (\bm{\theta} = \bm{\theta}_j), j = 1, \ldots, m$. Whereas other $g$-modeling methods often require nontrivial changes to existing algorithms to accommodate multivariate prior estimation \citep{WWStatComp2015}, we can simply use Algorithm \ref{alg:net-npmle} to solve \eqref{eq:net-npmle-phi}.



\begin{figure}[t]
    \centering
    \includegraphics[width=1.0\textwidth]{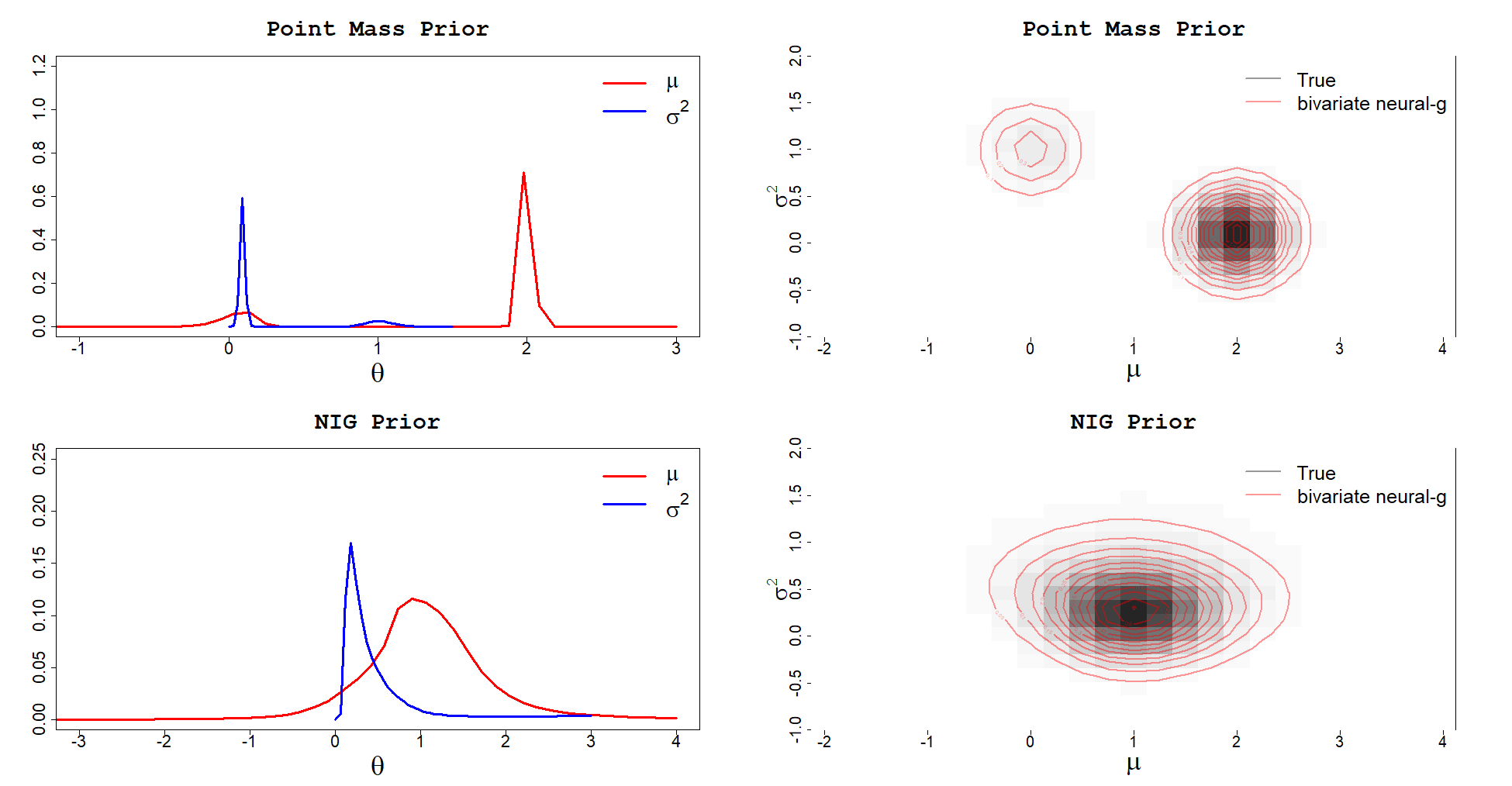}
\caption{Left two panels: Plots of the estimated marginal densities for $\mu$ and $\sigma^2$ under two different Gaussian location-scale mixture models. Right two panels: Heatmaps and contour plots from 5000 samples of the estimated joint density $\widehat{\pi}_{\text{neural-}g}(\mu, \sigma^2)$ vs. the true $\pi(\mu, \sigma^2)$.} \label{fig:bivari}
\end{figure}

To illustrate multivariate neural-$g$, we consider the Gaussian location-scale model studied by \cite{gu2017unobserved} and \cite{feng2018approximate}. Here, we have $\bm{\theta} = (\mu, \sigma^2) \in \mathbb{R}^{2}$ in \eqref{eq:multivariate-mixture}, and our interest lies in estimating the unknown bivariate prior density $\pi(\mu, \sigma^2)$. We simulated $n = 1000$ observations $\bm{y}_1, \ldots, \bm{y}_n$, where $\bm{y}_i = (y_{i1}, y_{i2})^\top$ for $i = 1, \ldots, n$, and 
\begin{equation}\label{eq:location-scale}
    y_{i1}, y_{i2} \mid (\mu_i,\sigma^2_i) \overset{i.i.d}{\sim} \mathcal{N}(\mu_i,\sigma^2_i),~~ (\mu_i, \sigma_i^2) \sim \pi(\mu, \sigma^2), ~~~~i = 1,\ldots, n.
\end{equation}
We considered the following two simulation settings for $\pi(\mu, \sigma^2)$:
\begin{itemize}
    \item \textbf{Discrete point mass prior}: $\pi(\mu,\sigma^2) \sim 0.2\delta_{\mu_1=0,\sigma_1^2=1.0}+0.8\delta_{\mu_2=2,\sigma_2^2=0.1}$. 
    \item \textbf{Normal-inverse-gamma (NIG) prior}: $\pi(\mu,\sigma^2) \sim \text{NIG}(\mu=1,\sigma=1,\text{shape}=2,\text{scale}=0.5)$.
\end{itemize}
The first setting is the example from \cite{feng2018approximate} where there is a discrete \emph{dependent} structure between $(\mu,\sigma^2)$.  Meanwhile,
the second setting has a continuous joint density for $\pi(\mu,\sigma^2)$. We used the algorithm in \cite{feng2018approximate} to first select the support grid $\bm{\Theta}_m$ with $m=50$. We then fit a bivariate neural-$g$  model to the data by solving \eqref{eq:net-npmle-phi-d}.



Our results are plotted in Figure \ref{fig:bivari}. The top two panels of Figure \ref{fig:bivari} plot the results for the discrete point mass prior. As expected, the estimated marginal densities of $\mu$ and $\sigma$ (upper left panel) under bivariate neural-$g$ exhibit discrete peaks at the point masses. The upper right panel of Figure \ref{fig:bivari} plots the contours for the estimated $\widehat{\pi}_{\text{neural-}g}(\mu, \sigma^2)$ against the heatmap of the true $\pi(\mu, \sigma^2)$ based on $5000$ samples. It is obvious that the bivariate neural-$g$ estimator matches the true bivariate density for $(\mu, \sigma^2)$ very well.

Results for the NIG prior are depicted in the two bottom panels of Figure \ref{fig:bivari}. The bottom left panel of Figure \ref{fig:bivari} clearly shows that the estimated marginal density for $\mu$ has a Gaussian shape centered at one, while the estimated marginal density for $\sigma^2$ has an inverse-gamma shape. The right panel of Figure \ref{fig:bivari} also shows that bivariate neural-$g$ accurately captures the true NIG density for $(\mu, \sigma^2)$. These examples show that multivariate neural-$g$ accurately captures both discrete \emph{and} continuous multivariate mixing densities where there is dependence between the elements of $\bm{\theta}$.

\section{Real Data Applications}\label{sec:real}


\subsection{Poisson Mixture Examples}\label{sec:surg}
The Poisson mixture model,
\begin{equation}\label{eq:poisson}
y_i\mid\lambda_i \sim \text{Poisson}(\lambda_i),~~
\lambda_i\sim \pi(\lambda),~~~~ i = 1, \ldots, n,
\vspace{-0.4cm}
\end{equation}
has been frequently employed to analyze count data  \citep{Karlis2018}. In \eqref{eq:poisson}, the latent density $\pi(\lambda)$ for the rate parameter $\lambda$ is our main parameter of interest.
In this section, we use neural-$g$ to analyze two count datasets.  In these two real applications, a single Poisson model (i.e. $y_i \mid \lambda \overset{iid}{\sim} \text{Poisson}(\lambda), i = 1, \ldots, n$, where $\lambda >0$ is fixed) is inadequate for modeling the data due to both overdispersion and clustering (i.e. individual random effects). A Poisson mixture model \eqref{eq:poisson} allows for much more flexibility, improved prediction, and identification of latent clustering effects. 
 
The first dataset contains the number of mutations at different protein domain positions \citep{park2022poisson}. Protein mutations are thought to be a primary cause of many degenerative and genetic disorders, and thus, it is of clinical importance to study their underlying mechanisms. These data were collected by The Cancer Genome Atlas (TCGA) from a total of 5848 patients and mapped to different protein domain positions \citep{park2022poisson}. In our analysis, we focused on three protein domains: growth factors (cd00031, $n$=366), protein kinases (cd00180, $n$=871) and the RAS-Like GT-Pase family of genes (cd00882, $n$=762). A detailed description of these data can be found in \cite{gauran2018empirical}. 

We fit Poisson mixture models \eqref{eq:poisson} with neural-$g$, NPMLE, Efron(5), and Efron(20) (i.e. Efron's $g$ with five or 20 degrees of freedom) to the protein data. We chose $\Theta_m$ as an equispaced grid of $m=100$ points from $y_{\min}$ to $y_{\max}$. Figure \ref{fig:protein} plots the estimated priors for protein domain cd00031. Appendix \ref{sec:addl-figures} provides plots of the results for protein domains cd00180 and cd00882. For cd00031, NPMLE shows four peaks at count values of 2, 5, 16, and 23. In contrast, both Efron(5) and Efron(20) show only three peaks around 0, 5, and 20. Efron's $g$ shows a Gaussian shape on the interval [15,25]. Neural-$g$ also suggests three clusters with local peaks at around 2 and 5 but has a \emph{uniform} mixture component on the interval [15, 25]. In this example, neural-$g$ gives a different interpretation from both NPMLE and Efron's $g$.

In our second real data application, we modeled the length of hospital stays in Arizona \citep{hilbe2011negative}.
This dataset consists of Arizona Medicare data from the year 1995 and records the lengths of hospital stays (in number of days) for $n=1798$ patients who had a coronary artery bypass graft (CABG) heart procedure. 
It is noteworthy that around 22\% of the patients in this dataset had counts equal to two, indicating the existence of inflation at two. Based on this fact, a well-fitted count model \eqref{eq:poisson} should be able to capture this characteristic from the dataset.

We fit Poisson mixture models \eqref{eq:poisson} with neural-$g$, NPMLE, Efron(5), and Efron(20) to the Arizona Medicare dataset where $\Theta_m$ was once again an equispaced grid of $m=100$ points from $y_{\min}$ to $y_{\max}$. Our results are plotted in Figure \ref{fig:azdrg112}. We observe that the NPMLE suggests that the data comes from a two-component discrete mixture model with rate parameters concentrated at 2 and 9. On the other hand, Efron(5) suggests that the true prior is a Gaussian-like density and fails to capture the inflation at two. Meanwhile, Efron(20) estimates a prior resembling a mixture of two Gaussians centered at 2 and 10.  Interestingly, neural-$g$ provides another viewpoint about the underlying mixing density $\pi(\lambda)$. Neural-$g$ suggests that the $\pi(\lambda)$ is a mixture of a point mass at 2 and a Gaussian centered at around 8.5.


\begin{figure}[t]
    \centering
    \begin{subfigure}[b]{0.48\textwidth}
    \includegraphics[width=\textwidth]{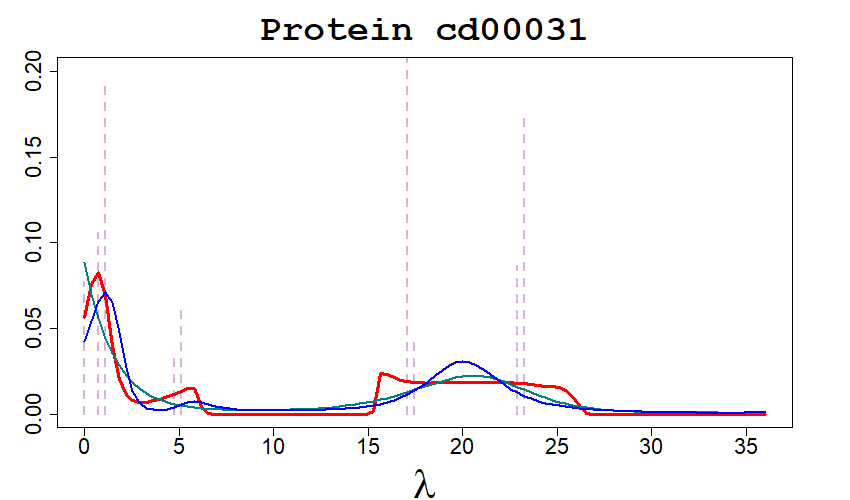}
    \caption{Protein cd00031 dataset}
    \label{fig:protein}
    \end{subfigure}
    \begin{subfigure}[b]{0.48\textwidth}
    \includegraphics[width=\textwidth]{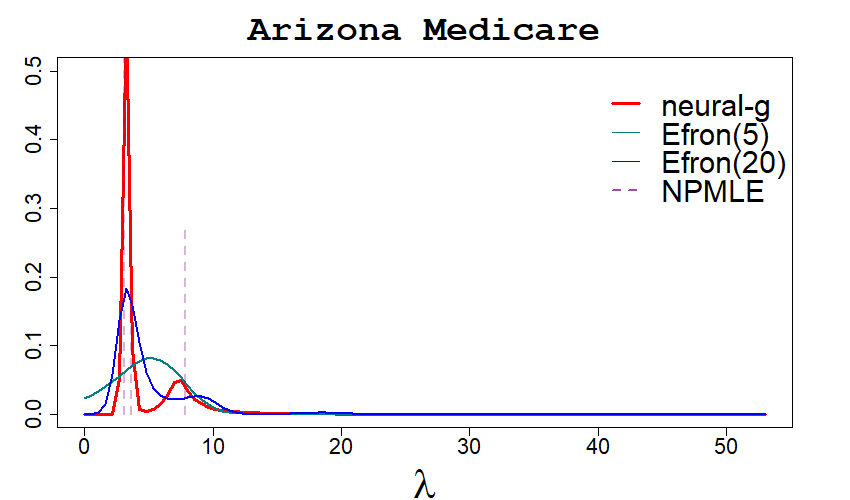}
    \caption{Arizona Medicare dataset}
    \label{fig:azdrg112}
    \end{subfigure}
\caption{Left panel: Estimated priors for the protein cd00031 dataset. Right panel: Estimated priors for the Arizona Medicare dataset.}
\end{figure}

To compare the model fits on these count datasets, we calculated $\chi^2$ statistics for comparing count frequencies. Specifically, we computed $\chi^2 = \sum_{c=1}^C|O_c-E_c|$ where $O_c$ is observed frequency while $E_c$ is expected frequency, defined as $E_c = n\sum_{j=1}^m f_y(c\mid \theta_j)\widehat{\pi}(\theta_j)$ for count $c$ under the estimated prior $\widehat{\pi}$ in \eqref{eq:poisson}. 
We used $K$-fold cross-validation (CV) with $K=10$ to evaluate the model fit, i.e. our metric was
\begin{equation} \label{eq:MAE} 
    \chi^2_{\text{MAE}} = \frac{1}{K} \sum_{k=1}^K\sum_{c_k=1}^{C_k} \bigg| O_{c_{k}}-n_k\sum_{j=1}^m f_y(c_k \mid \theta_j)\widehat{\pi}(\theta_j) \bigg|.
\end{equation}
To compare the predictive predictive power of the different methods, we also compared the out-of-sample predictive log-likelihood (PLL) from 10-fold CV as
\begin{equation} \label{eq:PLL}
    \text{PLL} = \frac{1}{K} \sum_{k=1}^K \sum_{I\in k}\sum_{j=1}^{m} - 2 \log f(y_{-I}; \theta_j \mid y_{I})\widehat{\pi}(\theta_j),
\end{equation}
where $y_{-I}$ is the held-out test data in the $k$th fold and $y_{I}$ denotes the rest of the data. In general, smaller $\chi^2_{\text{MAE}}$ and PLL suggest a better model fit.

Table \ref{tab:sim-cv-real} summarizes our two metrics \eqref{eq:MAE}-\eqref{eq:PLL} averaged across all 10 folds.
For the protein cd00031 dataset, Table \ref{tab:sim-cv-real} shows that the PLLs under four $\widehat{\pi}(\lambda)$ estimators were very similar; however, neural-$g$ achieved the lowest out-of-sample $\chi^2_{\text{MAE}}$. In addition, neural-$g$ had the lowest PLL in the Arizona Medicare dataset. One interesting thing to note is that different degrees of freedom led to greatly different model fits for Efron's $g$ in the Arizona Medicare dataset. This matches our observation in Figure \ref{fig:azdrg112}, which showed that Efron(5) estimated a unimodal prior centered around 5 (thereby missing the inflation at 2), whereas Efron(20) estimated a bimodal prior with a local mode centered at 2.


\renewcommand{\arraystretch}{1.5}
\begin{table}[t]
\centering
\caption{Results from our analyses of count datasets obtained using 10-fold CV.}
\label{tab:sim-cv-real}
\resizebox{0.9\columnwidth}{!}{%
\begin{tabular}{cc|cc|cc|cc|cc|cc}

\hline
&&\multicolumn{2}{c|}{neural-$g$}&\multicolumn{2}{c|}{{ NPMLE}}&\multicolumn{2}{c|}{{Efron(5)}}&\multicolumn{2}{c|}{{ Efron(20)}}\\
\hline
Dataset&$n$&$\chi^2_{\text{MAE}}$&PLL&$\chi^2_{\text{MAE}}$&PLL&$\chi^2_{\text{MAE}}$&PLL&$\chi^2_{\text{MAE}}$&PLL
\\
\hline
\hline
Protein cd00031&366&\textbf{18.27}&229.47&20.81&229.60&19.22&230.30&19.35&229.35
\\
Arizona Medicare&1789&49.54&\textbf{880.14}&50.74&880.21&48.47&903.17&50.35&887.76
\\
\hline
\end{tabular}}
\end{table}

\subsection{Measurement Error Example}\label{sec:framing}

Data collected from the real world are often contaminated by measurement errors. If we simply ignore the measurement error and treat the observations as if they were perfectly measured, this can greatly harm statistical inference. For example, we may fail to reveal the true underlying data density in density estimation \citep{efromovich1997density}, or hypothesis tests in regression analysis may be invalid \citep{hall2008measurement}. The additive measurement error model \citep{buonaccorsi2010measurement} is a classic model for handling this situation. 
Suppose that we observe contaminated data $\bm{y}= (y_1, \ldots, y_n)$. The additive measurement error model assumes that $\bm{y}$ were generated from the model,
\begin{equation*}
    y_i = \mu_i + \eta_i,~~~~ i = 1,\ldots, n,
\end{equation*}
where the $\mu_i$'s are the true (uncontaminated) values and the $\eta_i$'s are measurement errors. Assuming that $\mu_i \overset{iid}{\sim} \pi(\mu)$, the classic \textit{deconvolution} problem aims to estimate the density $\pi(\mu)$.

Here we revisit the Framingham Heart Study (FHS), one of the most famous investigations in cardiovascular disease epidemiology. See \cite{carroll2006measurement} for a detailed description of this study. The FHS measured the systolic blood pressure (SBP) for $n=1615$ patients during two examinations that were eight years apart (denoted as $\text{SBP}_1$ and $\text{SBP}_2$). In our analysis, we mainly focus on $\text{SBP}_2$. For both examinations, each $i$th individual patient had their blood pressure measured twice. We denote the two measurements of $\text{SBP}_2$ as $\bm{y}_i=(y_{i1}, y_{i2}), i=1, \ldots, 1615$.

We assume that $\text{SBP}_2$ follows an additive measurement error model, i.e.
\begin{equation} \label{eq:additive-2}
y_{ij} = \mu_i+\eta_i, ~~~~i = 1, \ldots, n,~~j=1, 2, 
\end{equation}
Our goal is to estimate the true density $\pi(\mu)$ of the \emph{uncontaminated} $\text{SBP}_2$.

We first followed the analysis of \cite{carroll2006measurement} and \cite{wang2011deconvolution} where the measurement errors $u_i$'s are assumed to be Gaussian and homogeneous.
By assuming homogeneity across patients, i.e. $\eta_i \overset{iid}{\sim} \mathcal{N}(0, \sigma^2)$, we have $y_{ij} \mid \mu_i \sim \mathcal{N}(\mu_i,\sigma^2)$ under model \eqref{eq:additive-2}. 
To estimate $\pi(\mu)$, we first took the average of the two measurements $(y_{i1}, y_{i2})$ so that $\overline{y}_i = 0.5(y_{i1}+y_{i2}) \sim \mathcal{N}(\mu_i,\sigma^2/2)$. Then an estimator of $\sigma^2$ could be obtained as
\begin{equation}\label{eq:fram-sigma}
    \widehat{\sigma}^2=\bigg( 0.5(n-1)^{-1} {\textstyle \sum_{i=1}^n }(y_i^*-\overline{y})^2 \bigg)^{-\frac{1}{2}},
\end{equation}
where $y^*_i = y_{1j}-y_{2j}$. Using the mixture model,
\begin{equation} \label{eq:fram-homogeneous}
    \overline{y}_i \mid \mu_i \sim \mathcal{N}(\mu_i, \widehat{\sigma}^2 / 2),~~ \mu_i \sim \pi(\mu),~~~~ i = 1, \ldots, n,
\end{equation}
where $\widehat{\sigma}^2$ is the plug-in estimator in \eqref{eq:fram-sigma}, we estimated the univariate mixing density $\pi(\mu)$ using neural-$g$, NPMLE, and Efron's $g$ with 5 or 20 degrees of freedom.

Instead of homogeneity, one can assume heterogeneous measurement errors for different patients. In this case, the additive measurement model \eqref{eq:additive-2} is $y_{ij} = \mu_i+ \eta_i$ where $\eta_i \overset{ind}{\sim} \mathcal{N}(0,\sigma^2_i)$. Then the model becomes the Gaussian location-scale mixture model \eqref{eq:location-scale} that we considered in Section \ref{sec:multivariate}, and our goal is to estimate the \emph{bivariate} density $\pi(\mu, \sigma^2)$ for the $\mu_i$'s and $\sigma_i^2$'s. We fit the bivariate neural-$g$ to this heterogeneous measurement error model.

\begin{figure}[t]
    \centering
    \includegraphics[width=1.0\textwidth]{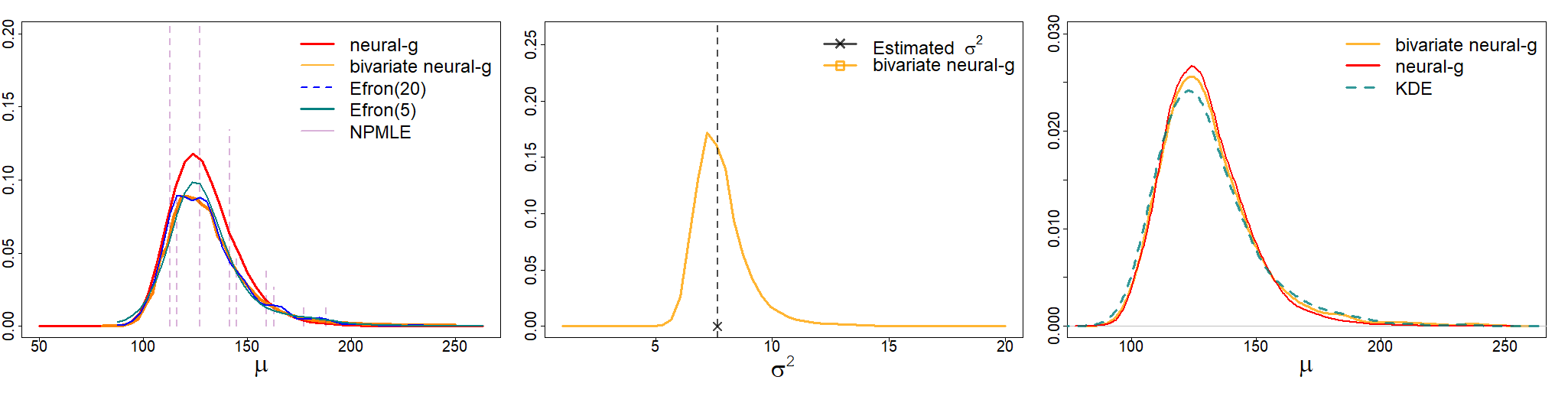}
\caption{\textbf{Left panel}: Estimated densities of $\pi(\mu)$ for the FHS data. \textbf{Middle panel}: Estimated marginal density for $\sigma^2$ in the FHS data under the bivariate neural-$g$ model. \textbf{Right panel}: Density estimates of uncontaminated $\text{SBP}_2$ levels based on 5000 samples of $\mu_i$'s from univariate neural-$g$ and bivariate neural-$g$ plotted against the kernel density estimate (KDE) of the observed (contaminated) $\text{SBP}_2$ levels $\bm{y}$.}\label{fig:fram}
\end{figure}

The left panel of Figure \ref{fig:fram} plots the estimated mixing densities $\pi(\mu)$ for univariate neural-$g$, NPMLE, Efron(5), and Efron(20) (i.e. under the assumption of homogeneous measurement errors), as well as the marginal density for $\mu$ under bivariate neural-$g$ (under the assumption of heterogeneous measurement errors). We observe that both neural-$g$ and Efron's $g$ estimators produced smooth estimates with a similar log-Gaussian shape. The NPMLE estimated an atomic density but its shape generally agreed with those of neural-$g$ and Efron's $g$. The middle panel of Figure \ref{fig:fram} plots the estimated marginal density for ${\sigma}^2$ under bivariate neural-$g$. We can see that the marginal density for $\sigma^2$ under the assumption of heterogeneous errors has a mode very close to the plug-in estimator $\widehat{\sigma}^2$ in \eqref{eq:fram-sigma} for the homogeneous measurement error model \eqref{eq:fram-homogeneous}. 

Finally, the right panel of Figure \ref{fig:fram} gives the kernel density plots based on 5000 samples of $\mu_i$ from $\widehat{\pi}(\mu)$ under univariate neural-$g$ and 5000 samples of $\mu_i$ from $\widehat{\pi}(\mu, \sigma^2)$ under bivariate neural-$g$. We compared these to a kernel density plot of the \emph{contaminated} data $\bm{y}$. We can see that the graph for the contaminated data has an obviously lower peak than the graphs for our two neural-$g$ estimates of the uncontaminated data. This illustrates the importance of accounting for measurement errors when inferring the density of the true $\text{SBP}_2$ levels in the FHS data.

\section{Conclusion}\label{sec:conclusion}

In this paper, we introduced neural-$g$, a new method for estimating a mixing density in $g$-modeling. Neural-$g$ uses a DNN with a softmax output layer to estimate the prior. Our method estimates a variety of priors well, including both smooth and non-smooth densities, thin-tailed and heavy-tailed densities, and flat and piecewise constant densities. Neural-$g$ occupies a happy medium between nonparametric approaches like NPMLE (which maximizes the total log-likelihood \eqref{eq:kwnpmle_sol} over \textit{all} PMFs defined on a finite parameter space $\Theta_m$) and parametric approaches like Efron's $g$ (which maximizes a penalized likelihood with respect to a parametric family on $\Theta_m$). Neural-$g$ is competitive with NPMLE when the true prior has only a few atoms and overcomes the limitations of the NPMLE \citep{koenker2020empirical} when the true prior is smooth. Meanwhile, neural-$g$ covers a wider class of densities than the curved exponential family of Efron's $g$ and avoids the challenge of selecting smoothing parameters \citep{koenker2019comment}.

Our approach is justified by a universal approximation theorem (Theorem \ref{theo:net-npmle}) showing neural-$g$'s capability to approximate any PMF arbitrarily well. To implement neural-$g$, we introduced the WAG optimizer, a variant of SGD which averages over previous epochs in addition to previous batch gradients. We further proposed multivariate neural-$g$ for estimating multivariate priors. We demonstrated our framework through a wide variety of simulations and real data applications.

There are a few extensions and improvements for future research.  First, it will be worthwhile to extend neural-$g$ to regression models with covariates such as generalized linear mixed models \citep{MagderZeger1996}, finite mixture of regression models \citep{jiang2021nonparametric}, and conditional density estimation \citep{izbicki2017converting, wang2024PGQR}. Secondly, our method -- like most other existing $g$-modeling approaches -- requires specification of a grid \textit{a priori} for estimation. An implementation of neural-$g$ that can either automatically select the suitable grid points or that does \emph{not} require ad-hoc discretization is of great interest.
\bibliographystyle{apalike}
    \bibliography{ref}


\newpage

\appendix

\section{Proofs of Theoretical Results} \label{sec:proofs}

\subsection{Proof of Theorem \ref{theo:net-npmle}}\label{sec:theo-net-npmle}




\begin{proof}
Note that we constructed $g_{\phi} \in \mathcal{G}_m$ in \eqref{eq:MLP-family} as a single-input and single-output (SISO) MLP which specifically maps a given $\theta_j \in \mathbb{R}^{1}$ to $[0,1]^{1}$. Moreover, a PMF function on discrete space $\Theta_m = \{ \theta_1, \ldots, \theta_m \}$ can be viewed as a composition of $m$ univariate single-input and single-output (SISO) functions with the constraint that their outputs sum to one. \cite{zeng2005approximation} established the universal approximation capability of multiple-input and single-ouput (MISO) MLPs. Insofar as a SISO MLP is a special case of a MISO MLP (with only one input), we can extend the result of \cite{zeng2005approximation} to neural-$g$.


Let $h_{\epsilon} = \log(2m \cdot p(\theta)/\epsilon+1)$. Since $m < \infty$, $h_{\epsilon}$ is a well-defined function on $\Tm$.
Let $h_{\epsilon}(\Tm)=[h_{\epsilon}(\theta_1),\ldots,h_{\epsilon}(\theta_m)]^\top$, and let $\text{softmax}(h_{\epsilon}(\Theta_m))$ denote the $m \times 1$ vector whose $j$th entry is $\exp(h_{\epsilon}(\theta_j)) / \sum_{j=1}^{m} \exp(h_{\epsilon}(\theta_j))$. Using the triangle inequality, we have 
\begin{align*}  \label{eq:triangle-inequality}
    \big\Vert g_{\phi}(\Tm) -p(\Tm)\big\Vert_{\infty}~ 
    &\leq~\big\Vert  g_{\phi}(\Tm) -\softmax(h_{\epsilon}(\Tm))\big\Vert_{\infty} + \big\Vert\softmax(h_{\epsilon}(\Tm))-p(\Tm)\big\Vert_{\infty} \\
    &\overset{\Delta}{=}~\Delta_1 + \Delta_2 \numbereqn
\end{align*}
We first consider $\Delta_1$. It is known that the $\softmax$ function is a continuous function on its domain, and $g_{{\phi}}(\Tm)$ is constructed with a softmax output layer. Let $v_{\phi}(\Tm) = [v_{\phi}(\theta_1), \ldots, v_{\phi}(\theta_m)]^\top$ be such that $g_{\phi}(\Tm) = \softmax(v_{\phi}(\Tm))$. By the definition of a continuous function, we have that for any $\epsilon/2> 0$, there exists a $\delta > 0$ such that
\begin{equation*}
    \text{If } \Vert v_{\phi}(\Tm)-h_{\epsilon}(\Tm) \Vert_{\infty} < \delta,~\text{ then }~  \big\Vert\softmax(v_{\phi}(\Tm))-\softmax(h_{\epsilon}(\Tm))\big\Vert_{\infty} < \epsilon/2.
\end{equation*}
According to Theorem 2 in \cite{zeng2005approximation}, for any given $\delta>0$, there exists a single-hidden-layer MISO MLP $M$ on $\Tm$ such that $\Vert M(\Tm)-h_{\epsilon}(\Tm)\Vert_{\infty} < \delta$. 
Since $g_{\phi}$ is constructed as a SISO MLP which is a special case of a MISO MLP, it is immediate that $\Vert v_{\phi}(\Tm)-h_{\epsilon}(\Tm)\Vert_{\infty} < \delta$.
Then for any $\epsilon > 0$, there exists a single-hidden-layer SISO MLP $g_{\phi} \in \mathcal{G}_m$ satisfying 
\begin{equation} \label{eq:Delta-1}
    \Delta_1 = \big\Vert
    g_{\phi}(\Theta_m) -\softmax(h_{\epsilon}(\Tm))\big\Vert_{\infty} < \epsilon/2.
\end{equation}
Secondly, for $\Delta_2$, we have 
\begin{align*} \label{eq:Delta-2}
\Delta_2 &= \big\Vert\softmax(h_{\epsilon}(\Tm))-p(\Tm)\big\Vert_{\infty} \\
&= \underset{k \in \{ 1,\ldots, m \}}{\max} \bigg\vert\frac{\exp(\log(2m \cdot p(\Tm)_k/\epsilon+1))}{\sum^m_{j=1} \exp(\log(2m \cdot p(\Tm)_j/\epsilon+1))}-p(\Tm)_k \bigg\vert\\
&=\underset{k\in \{ 1,\ldots, m \}}{\max} \bigg\vert \frac{2m \cdot p(\Tm)_k/\epsilon+1}{2m\sum^m_{j=1}p(\Tm)_j/\epsilon+m}-p(\Tm)_k \bigg\vert\\
&=\underset{k\in \{1,\ldots, m\}}{\max} \bigg\vert \frac{2m \cdot p(\Tm)_k/\epsilon+1}{2m/\epsilon+m}-p(\Tm)_k \bigg\vert\\
&=\underset{k\in \{1,\ldots, m\}}{\max} \bigg\vert \frac{2m \cdot p(\Tm)_k/\epsilon+1-2m \cdot p(\Tm)_k/\epsilon-m \cdot p(\Tm)_k}{2m/\epsilon+m}\bigg\vert\\
&\leq\frac{1}{2m/\epsilon+m} \bigg[ \underset{k: p(\Tm)_k\neq0}{\max} \vert 1-m \cdot p(\Tm)_k \vert +  \underset{k: p(\Tm)_k=0}{\max} \vert 1-m \cdot p(\Tm)_k\vert \bigg]\\
&=\frac{m}{2m/\epsilon+m} \bigg[ \underset{k: p(\Tm)_k\neq0}{\max} \big| \underbrace{\frac{1}{m}-p(\Tm)_k}_{\leq 1-\frac{1}{m}} \big|  +  \underset{k: p(\Tm)_k=0}{\max} \big| \frac{1}{m}-k\cdot 0 \big| \bigg] \\
&\leq\frac{m}{2m/\epsilon+m}, \numbereqn \\
\end{align*}
where $p(\Theta_m)_{k}$ denotes the $k$th entry of $p(\Theta_m)$. Combining \eqref{eq:triangle-inequality}-\eqref{eq:Delta-2}, we have
\begin{equation*}
\lVert g_{\phi}(\Tm) - p(\Tm)\big\Vert_{\infty}~
\leq~  \Delta_1 + \Delta_2  ~<~ \frac{\epsilon}{2}+\frac{m}{2m/\epsilon+m}   ~=~ \frac{\epsilon}{2}+ \frac{1}{1+\frac{2}{\epsilon}} < \frac{\epsilon}{2} + \frac{1}{2/\epsilon}  = \epsilon.
\end{equation*}
The proof is done. 
\end{proof}

\subsection{Proof of Theorem \ref{theo:net-npmle-2}}\label{sec:theo-net-npmle-2}

\begin{proof}
	By Assumption (A2), the global minimizer $\widehat{\pi}_m$ for the loss function \eqref{eq:optim-over-Thetam} exists. Let $\widehat{\pi}_m(\Theta_m) = ( \widehat{\pi}_m(\theta_1), \ldots, \widehat{\pi}_m(\theta_m ))^\top$, where the $j$th entry of $\widehat{\pi}_m(\Theta_m)$ is the probability assigned to $\theta_m$ by $\widehat{\pi}_m$. Further, let $\pi_{\text{neural-}g} = \sum_{j=1}^{m} g_{\phi}(\theta_j) \delta_{\theta_j}$, where $g_{\phi} \in \mathcal{G}_m$ is a single-hidden-layer network in the MLP family \eqref{eq:MLP-family}, such that
	\begin{equation} \label{eq:pi-g}
		\lVert \widehat{\pi}_m (\Theta_m) - g_{\phi}(\Theta_m ) \rVert_{\infty} < \frac{c_1 \epsilon}{m c_2},
	\end{equation}
	where $c_1$ and $c_2$ are positive constants to be specified in the next two paragraphs. By Theorem \ref{theo:net-npmle}, such a network $g_{\phi}$ in \eqref{eq:pi-g} must exist.
	
	To ease the notation, let $L_i(\pi) = \sum_{j=1}^{m} f(y_i \mid \theta_j) \pi(\theta_j)$ for a generic prior $\pi$. Then the loss function \eqref{eq:optim-over-Thetam} can be written as $\ell(\pi) = -n^{-1} \sum_{i=1}^{n} \log L_i(\pi)$.
	Assumption (A2) implies that $L_i(\widehat{\pi}_m) > 0$ for all $i = 1, \ldots, n$, because if not, then $\ell(\widehat{\pi}_m)$  would either be infinite or undefined, and $\widehat{\pi}_m$ would not be a global minimizer. Further, since $g_{\phi}$ has a softmax output layer, it must be that $g_{\phi} (\theta_j) > 0$ for all $\theta_j \in \Theta_m$. By Assumption (A1) that $f(y_i \mid \theta) > 0$ for at least one $\theta \in \Theta_m$ and all $i = 1, \ldots, n$, it follows that $L_i(\pi_{\text{neural-}g}) = \sum_{j=1}^{m} f(y_i \mid \theta_j) g_{\phi}(\theta_j) > 0$ for all $i = 1, \ldots, n$.  Thus, for some constant $c_1 > 0$, we have
	$c_1 \leq \min \big\{ L_i(\widehat{\pi}_m), L_i(\pi_{\text{neural-}g}) \big\}$ for all $i = 1, \ldots, n$. By Lipschitz continuity of the log function on the interval $[ c_1, \infty)$, 
	\begin{align} \label{eq:Lipschitz}
		| \log L_i(\widehat{\pi}_m) - \log L_i (\pi_{\text{neural-}g}) | \leq \frac{1}{c_1} \big| L_i(\widehat{\pi}_m) - L_i(\pi_{\text{neural-}g}) \big|.  
	\end{align}
	Meanwhile, we also have by Assumption (A1) that there exists a constant $c_2 > c_1$ such that 
	\begin{align} \label{eq:max-c2}
		\max_{1 \leq i \leq n, 1 \leq j \leq m} f(y_i \mid \theta_j) \leq c_2 < \infty.
	\end{align}
	It follows that
	\begin{align*} \label{eq:ell-diff}
		\big| \ell(\pi_{\text{neural-}g}) - \ell(\widehat{\pi}_m) \big| & = \bigg| \frac{1}{n} \sum_{i=1}^{n}  \big[ \log L_i(\widehat{\pi}_m) - \log L_i (\pi_{\text{neural-}g}) \big] \bigg| \\
		& \leq \frac{1}{n} \sum_{i=1}^{n} \big| \log L_i(\widehat{\pi}_m) - \log L_i(\pi_{\text{neural-}g}) \big| \\
		& \leq \frac{1}{c_1 n} \sum_{i=1}^{n} \big| L_i(\widehat{\pi}_m) - L_i(\pi_{\text{neural-}g}) \big| \\
		& \leq \frac{1}{c_1 n} \sum_{i=1}^{n} \sum_{j=1}^{m} \bigg| f(y_i \mid \theta_j) \big[ \widehat{\pi}_m(\theta_j) - \pi_{\text{neural-}g}(\theta_j) \big] \bigg| \\
		& \leq \frac{c_2}{c_1} \sum_{j=1}^{m} \big| \widehat{\pi}_m(\theta_j) - \pi_{\text{neural-}g}(\theta_j) \big| \\
		& \leq \frac{m c_2}{c_1} \lVert \widehat{\pi}_m(\Theta_m) - g_{\phi}(\Theta_m)  \rVert_{\infty} \\
		& < \frac{m c_2}{c_1} \cdot \frac{c_1 \epsilon}{m c_2} = \epsilon, \numbereqn
	\end{align*}
	where we used \eqref{eq:Lipschitz} in the third line, \eqref{eq:max-c2} in the fifth line, and \eqref{eq:pi-g} in the last line. The proof is done.
\end{proof}

\section{Sensitivity Analysis}\label{sec:sensitivity}

In this section, we investigate neural-$g$'s robustness to the hyperparameters in the neural network architecture, i.e. the number of hidden layers $L$ and the number of neurons $h$ in each hidden layer. 
In the first experiment, we fixed $h= 500$ and varied $L \in \{ 1, 2, 3, 4, 5\}$. In the second experiment, we fixed $L=5$ and varied $h \in \{ 200, 500, 800, 1200, 2000 \}$. We repeated  all of the simulations from Section \ref{sec:simulate} using different combinations of $(L, h)$. 

Our results averaged across 10 replications are reported in Tables \ref{tab:NN-depth-sim} (for fixed width) and Table \ref{tab:NN-width-sim} (for fixed depth). These results demonstrate that neural-$g$ is fairly robust to the specific choices of $L$ and $h$. In Simulations I-IV (i.e. uniform, piecewise constant, heavy-tailed, and bounded priors), neural-$g$ still vastly outperformed NPMLE and Efron's $g$ (Table \ref{tab:m1-m3}), with lower average Wasserstein-1 discrepancy and average MAE of the Bayes estimators, regardless of $L$ or $h$. From Table \ref{tab:NN-depth-sim}, we also observe that when $L$ was greater than one, the performance of neural-$g$ improved slightly. However, even with a single hidden layer, neural-$g$ still generally performed well.



\vspace{0.2cm}
\renewcommand{\arraystretch}{1.1}
\begin{table}[h]
\caption{Simulation results for neural-$g$ with width $h=500$ and number of hidden layers $L \in \{1, 2, 3, 4, 5\}$. The displayed results are averaged over 10 replications.}
\centering
\resizebox{1.0\columnwidth}{!}{%
\begin{tabular}{c|cc|cc|cc|cc|cc|}
\hline
&\multicolumn{2}{c|}{$L=1$}&\multicolumn{2}{c|}{$L=2$}&\multicolumn{2}{c|}{$L=3$}&\multicolumn{2}{c|}{$L=4$}&\multicolumn{2}{c|}{$L=5$}\\
\hline
Prior $\pi$&$W_1(\pi,\widehat{\pi}_{\bm{t}})$&$\text{MAE}$&$W_1(\pi,\widehat{\pi}_{\bm{t}})$&$\text{MAE}$&$W_1(\pi,\widehat{\pi}_{\bm{t}})$&$\text{MAE}$&$W_1(\pi,\widehat{\pi}_{\bm{t}})$&$\text{MAE}$&$W_1(\pi,\widehat{\pi}_{\bm{t}})$&$\text{MAE}$
\\
\hline
\hline
Uniform&0.109&0.051&0.072&0.044&0.067&0.035&0.060&0.030&0.061&0.030\\
Piecewise constant&0.075&0.029&0.061&0.029&0.047&0.029&0.054&0.030&0.052&0.030\\
Heavy-tailed&0.073&0.031&0.057&0.022&0.051&0.028&0.055&0.028&0.053&0.025\\
Bounded&0.032&0.008&0.031&0.005&0.031&0.005&0.031&0.005&0.031&0.005\\
Point Mass&4.641&0.049&4.579&0.036&4.547&0.026&4.555&0.023&4.553&0.017\\
Gaussian&0.103&0.066&0.075&0.044&0.057&0.029&0.055&0.025&0.04&0.014\\
\hline
\end{tabular}}
\label{tab:NN-depth-sim}
\end{table}
\renewcommand{\arraystretch}{1.1}
\begin{table}[h!]
\caption{Simulation results for neural-$g$ with five hidden layers and width $h \in \{200, 500, 800, 1200, 2000 \}$. The displayed results are averaged over 10 replications.}
\centering
\resizebox{1.0\columnwidth}{!}{%
\begin{tabular}{c|cc|cc|cc|cc|cc|}
\hline
&\multicolumn{2}{c|}{$h=200$}&\multicolumn{2}{c|}{$h=500$}&\multicolumn{2}{c|}{$h=800$}&\multicolumn{2}{c|}{$h=1200$}&\multicolumn{2}{c|}{$h=2000$}\\
\hline
Prior $\pi$&$W_1(\pi,\widehat{\pi}_{\bm{t}})$&$\text{MAE}$&$W_1(\pi,\widehat{\pi}_{\bm{t}})$&$\text{MAE}$&$W_1(\pi,\widehat{\pi}_{\bm{t}})$&$\text{MAE}$&$W_1(\pi,\widehat{\pi}_{\bm{t}})$&$\text{MAE}$&$W_1(\pi,\widehat{\pi}_{\bm{t}})$&$\text{MAE}$
\\
\hline
\hline
Uniform&0.066&0.021&0.059&0.025&0.064&0.023&0.068&0.021&0.074&0.039\\
Piecewise constant&0.051&0.03&0.057&0.03&0.059&0.03&0.054&0.03&0.057&0.03\\
Heavy-tailed&0.052&0.028&0.055&0.030&0.056&0.028&0.052&0.025&0.056&0.027 \\
Bounded&0.031&0.007&0.030&0.005&0.031&0.005&0.031&0.005&0.031&0.005 \\
Point Mass &4.533&0.024&4.517&0.024&4.566&0.025&4.522&0.024&4.595&0.025\\
Gaussian&0.049&0.027&0.048&0.025&0.046&0.019&0.080&0.048&0.041&0.015\\
\hline
\end{tabular}}
\label{tab:NN-width-sim}
\end{table}


\section{Additional Figures}\label{sec:addl-figures}
In Section \ref{sec:real}, we analyzed mutation counts on the protein domain cd00031. Here, we provide the plots for the estimated prior densities $\pi(\lambda)$ for protein domains cd00180 and cd00882 in Figures \ref{fig:protein2} and \ref{fig:protein3}. As shown in these figures, neural-$g$ gives a very similar result to Efron's $g$ for cd00180 and cd00882. This is in contrast to cd00031 (Figure \ref{fig:protein}) where the neural-$g$ estimator departed from Efron's $g$, particularly for counts between 15 and 25. 

\begin{figure}[H]
    \centering
    \begin{subfigure}[b]{0.48\textwidth}
    \includegraphics[width=\textwidth]{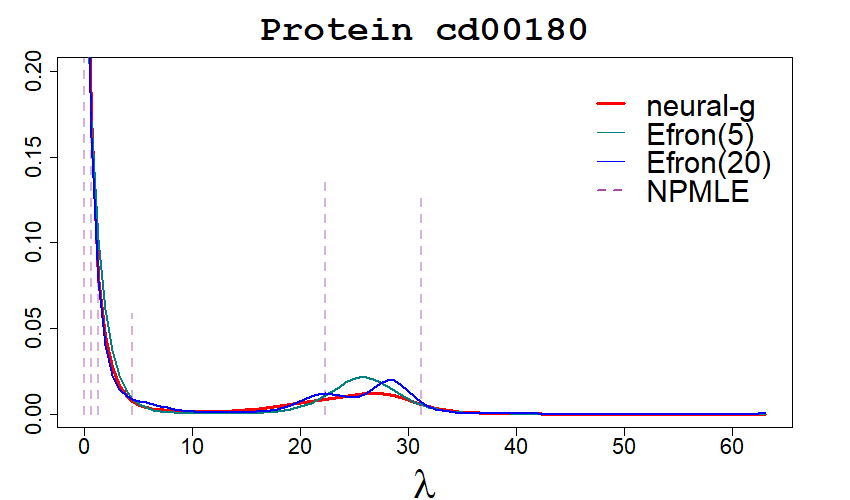}
    \caption{Protein cd00180 dataset}
    \label{fig:protein2}
    \end{subfigure}
    \begin{subfigure}[b]{0.48\textwidth}
    \includegraphics[width=\textwidth]{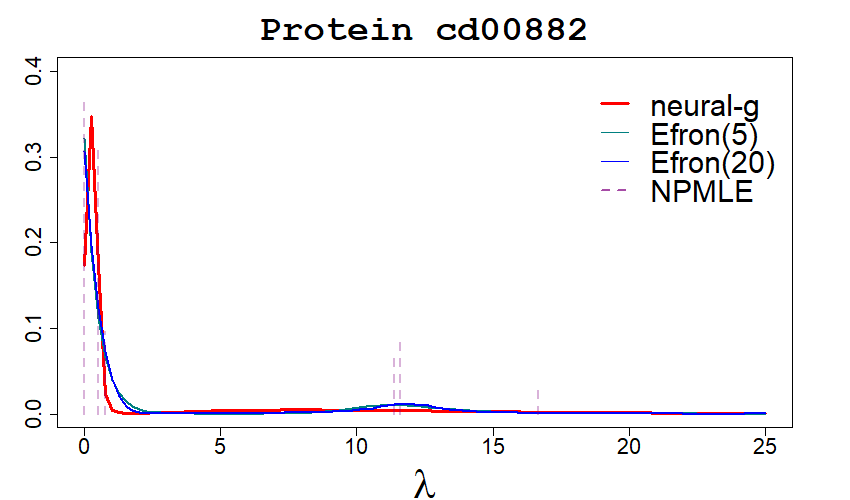}
    \caption{Protein cd00882 dataset}
    \label{fig:protein3}
    \end{subfigure}
\caption{Estimated priors for the protein domain cd00180 and protein domain cd00882 datasets.}
\end{figure}

\end{document}